\documentclass{article}
\usepackage[margin=1in]{geometry}
\usepackage{mathtools}
\usepackage{amssymb}
\usepackage{amsthm}
\usepackage{enumitem}
\usepackage[round]{natbib}
\usepackage{bm,mathrsfs}
\usepackage{microtype}
\usepackage{xcolor}
\usepackage{graphicx}
\usepackage{float}
\usepackage{booktabs}
\usepackage{soul}
\usepackage[colorlinks=true,linkcolor=blue,citecolor=blue,urlcolor=blue]{hyperref}

\theoremstyle{plain}
\newtheorem{theorem}{Theorem}[section]
\newtheorem{proposition}{Proposition}[section]
\newtheorem{lemma}{Lemma}[section]
\newtheorem{corollary}{Corollary}[section]
\newtheorem{assumption}{Assumption}[section]

\theoremstyle{definition}
\newtheorem{definition}{Definition}[section]

\theoremstyle{remark}
\newtheorem*{remark}{Remark}

\NewDocumentCommand{\eqs}{s m}{
  \IfBooleanTF{#1}{
    \begin{equation*}
    \begin{aligned}
    #2
    \end{aligned}
    \end{equation*}
  }{
    \begin{equation}
    \begin{aligned}
    #2
    \end{aligned}
    \end{equation}
  }
}

\DeclareMathOperator*{\argmin}{arg\,min}

\newcommand{\Cov}{\operatorname{Cov}}

\newcommand{\1}{\mathbf{1}}
\newcommand{\R}{\mathbb{R}}
\newcommand{\E}{\mathbb{E}}

\newcommand{\Fcal}{\mathcal{F}}

\newcommand{\diag}{\mathrm{diag}}

\newcommand{\ind}[1]{\mathbf{1}\{#1\}}

\newcommand{\N}{\mathcal{N}}

\DeclareMathOperator{\dist}{dist}

\begin{document}

\title{A Nonlinear Time-Series Theory for Constant-step SGD with Momentum}
\author{
    Younghun Lee\textsuperscript{1}, Jake Zummo\textsuperscript{1}, and Wei Biao Wu\textsuperscript{1}\\[0.5em]
    \textit{\textsuperscript{1}Department of Statistics, University of Chicago}
}
\date{October 11, 2026}

\maketitle

\begin{abstract}
Constant-step stochastic gradient descent with momentum (SGDM) augments the parameter recursion with a momentum state that carries geometrically decaying memory of past stochastic gradients. We study this coupled recursion as a nonlinear time series and develop a framework for its stationary behavior and distributional limits. The joint parameter-momentum process forgets its initialization geometrically under a finite $p$th moment for $p > 1$, yielding a unique stationary causal solution. Under a second-moment condition, this contraction leads to an invariance principle and quenched Gaussian limit for averaged SGDM, while a small-learning-rate analysis gives a Gaussian limit for the last iterate. In a stochastic-curvature quadratic model, we further study how the rate of initialization forgetting relates to stationary variance and averaged fluctuations.
\end{abstract}

\section{Introduction}
\label{sec:introduction}

Momentum modifies the dynamics of gradient descent by incorporating past descent directions into each update. Classical heavy-ball and Nesterov methods demonstrate how this mechanism can accelerate deterministic optimization \citep{polyak1964some,nesterov1983method}. Momentum is also widely used in stochastic gradient descent (SGD) \citep{bottou2018optimization}, including neural-network training, where suitable initialization and momentum scheduling can substantially improve training \citep{sutskever2013on}. In SGD with momentum (SGDM), this modification also changes how gradient noise propagates: past noise affects the current parameter and also enters its update through the accumulated momentum. Understanding the resulting stability and fluctuations calls for a joint analysis of the parameter and momentum sequences.

At a constant learning rate, gradient noise can persist in the long run, motivating a description in terms of stationary behavior and distributional limits, as developed for SGD \citep{pflug1986stochastic,dieuleveut2020bridging,Chen2022Stationary}. For momentum, \citet{can2019accelerated} establish accelerated convergence to a stationary law for additive-noise quadratics. Fluctuation limits have also been obtained for decreasing-step stochastic heavy ball \citep{gadat2018stochastic} and for averaged SGDM with diverging mini-batch sizes at a fixed learning rate \citep{tang2023acceleration}. For nonlinear SGDM with a fixed learning rate and sampling distribution, the question remains how to connect stability of the coupled dynamics to dependence control and averaging limits around the stationary mean.

We develop a common framework for the stability and limit theory of constant-step SGDM, building on the SGD analysis of \citet{li2026stochastic} and nonlinear time-series methods \citep{chen2016stability,Wu2011Asymptotic}. For strongly convex population objectives with state-dependent gradient noise, geometric moment contraction (GMC) of the joint parameter--momentum state provides the link from stability to dependence control and averaging limits. 
Under the local Hessian condition and sufficient burn-in, we also obtain a Gaussian limit for the last iterate as the constant step tends to zero at fixed $\beta$, following the ordinary-SGD setting of \citet{wei2025gaussian}. A complementary analysis of quadratics with commuting random curvature examines when momentum accelerates forgetting and how this relates to stationary and averaged fluctuations.

\subsection{Related work}
Geometric contraction connects stationary laws with dependence and limit theory for nonlinear time series \citep{diaconis1999iterated,chen2016stability,Wu2011Asymptotic}. \citet{li2026stochastic} apply this approach to constant-step SGD, obtaining dependence bounds and averaging limits. For momentum, \citet[Theorem~11]{can2019accelerated} prove Wasserstein convergence to a stationary law for additive-noise quadratics at the optimal heavy-ball tuning. We extend the nonlinear SGD approach to the joint parameter and momentum state, with contraction conditions allowing state-dependent noise and finite moments of order $p>1$ at each fixed $\beta<1$.

For additive-noise quadratic momentum, \citet{gitman2019understanding} analyze stationary covariance and its dependence on the learning rate and momentum. \citet{Mou2020On} establish an averaged CLT for fixed-step linear stochastic approximation, including momentum, with a multiplicative-noise covariance correction that also appears in our quadratic variance calculation. They also analyze how momentum affects mixing and the covariance correction. \citet{tang2023acceleration} obtain asymptotic normality for averaged SGDM centered at the minimizer, with diverging mini-batch sizes when the learning rate is fixed. Our averaging results keep the sampling distribution fixed and are centered at the stationary mean. \citet{gadat2018stochastic} and \citet{li2023revisiting} establish momentum CLTs with decreasing step sizes. We instead consider constant learning rate and momentum coefficient and study the last iterate as learning rate tends to zero.

\citet{kidambi2018insufficiency} construct least-squares instances on which no heavy-ball tuning improves the order of SGD. \citet{lee2022trajectory} characterize the convergence rate and limiting loss of constant-step momentum in high-dimensional least squares. For quadratic problems, \citet{bollapragada2025fast} and \citet{dang2024noiseadaptive} give bounds that separate accelerated decay of initialization error from the error level due to gradient noise, under suitable mini-batch conditions. For commuting random curvature, we derive lower bounds on averaged variance for prescribed forgetting rates, give an accelerated tuning under small curvature variance, and bound finite-average mean-square error.

\subsection{Contributions}

We develop a nonlinear time-series framework for the stability and limit theory of constant-step SGDM.

\begin{enumerate}
\renewcommand{\theenumi}{(\roman{enumi})}
\renewcommand{\labelenumi}{\theenumi}
    \item For nonlinear SGDM, we establish a unique stationary causal solution and geometric forgetting of initialization. We prove joint-state $L^p$ GMC for $p>1$ under explicit conditions at each fixed $0\le\beta<1$, allowing infinite variance when $1<p<2$. This contraction also yields geometric decay of functional dependence.
    \item We develop distributional limit theory for averaged and last-iterate SGDM under second moments. At fixed \((\gamma,\beta)\), GMC yields averaging rates, an invariance principle, and a quenched CLT around the stationary mean. Under local gradient linearization, we also establish a quenched Gaussian limit for the last iterate centered at the minimizer as $\gamma\to0$ at fixed $\beta$, after sufficient burn-in.
    The limiting covariance is $\Gamma/(1-\beta)$ on the $\sqrt{\gamma}$ scale.
    \item For quadratics with commuting random curvature, accelerated initialization forgetting requires $\nu=O(\kappa^{-1/2})$ under the curvature and noise conditions of Theorem~\ref{thm:rate-variance-cost}(i). It also entails an $\Omega(\nu\sqrt{\kappa})$ relative increase in averaged limiting variance over the fixed-curvature reference in at least one direction. Under the additional conditions in part~(ii), we give an accelerated tuning with mean-square error bounds for finite averages.
\end{enumerate}

\subsection{Notations}
We use $|\cdot|$ for Euclidean and induced operator norms, $\langle v,w\rangle=v^\top w$, and $\|X\|_p=(\mathbb E|X|^p)^{1/p}$. We write $M^\top$ and $\rho(M)$ for transpose and spectral radius respectively, and $a\wedge b=\min(a,b)$. For symmetric matrices, $M\preceq N$ means $N-M$ is positive semidefinite. We use standard $O/o$ notation; $a\lesssim b$ means $a\le Cb$ and $a\asymp b$ (or $a=\Theta(b)$) means two-sided bounds. Constants may vary and are independent of the limiting variable. The subscripts $\mathbb P$ and ${\rm a.s.}$ indicate probability and almost-sure bounds, and $X_n=O_{L^p}(b_n)$ means $\|X_n\|_p=O(b_n)$. We write $\Rightarrow$ for convergence in distribution and $N(0,\Sigma)$ for a centered Gaussian law with covariance matrix $\Sigma$.

\section{Momentum as a nonlinear AR}
\label{sec:setup}

Let $(x_k)_{k\in\mathbb Z}$ be i.i.d. random elements on a measurable space, which we call innovations. We consider the problem of minimizing the risk function $F(\theta) = \E[f(\theta, x_0)]$ for a sample loss $f(\theta, x)$ parameterized by $\theta \in \R^d$. Define $\theta^\star = \argmin_{\theta \in \R^d} F(\theta)$; we later assume various strong-convexity conditions on $F$ (see Assumption \ref{ass:global} below) to ensure that $\theta^\star$ is uniquely defined. Let $G(\theta) = \nabla_\theta F(\theta)$ and $g(\theta, x) = \nabla_\theta f(\theta, x)$ denote the population and stochastic gradients, respectively. Since the parameter space is unconstrained, first-order conditions imply that $G(\theta^\star) = 0$.

For an initialization $(\theta_0, y_0)$ of parameter and momentum, define $\mathcal{F}_k = \sigma(\theta_0, y_0, x_1, \dots, x_k)$ to be the natural filtration for the SGDM iterates. For the two-sided innovation sequence, write $\mathcal{X}_k=(\ldots,x_{k-1},x_k)$ and let $\mathcal{F}_k^\circ=\sigma(x_j:j\le k)$. For integrable $X$, define the martingale projection $\mathcal{P}_kX=\mathbb{E}[X\mid\mathcal{F}_k^\circ]-\mathbb{E}[X\mid\mathcal{F}_{k-1}^\circ]$.

For this estimation problem, we study SGD with momentum with constant momentum factor $\beta\in[0,1)$ and learning rate $\gamma>0$:
\begin{equation}\label{eq:sgdm}
\begin{aligned}
    y_k &= \beta y_{k-1} - \gamma g(\theta_{k-1},x_k), \\
    \theta_k &= \theta_{k-1}+y_k. 
\end{aligned}
\end{equation}
Here $\theta_k$ is the parameter iterate, and $y_k \in \R^d$ is the momentum sequence. We write $\eta=\gamma/(1-\beta)$ for the normalized learning rate. We can interpret $y_k$ as a $\beta$-weighted exponential moving average of $-\eta g(\theta_{k-1}, x_k)$.

Set $\theta_{-1}:=\theta_0-y_0$.
Eliminating $y_k$  in \eqref{eq:sgdm} gives the nonlinear AR(2) recursion for $k\ge1$,
\begin{equation}\label{eq:nonlinear-ar2}
    \theta_k= (1+\beta)\theta_{k-1}-\beta\theta_{k-2} -\gamma g(\theta_{k-1},x_k).
\end{equation}
This representation places SGDM within the nonlinear autoregressive setting of \citet{chen2016stability}. Under the GMC conditions in Section~\ref{sec:gmc}, the recursion admits a unique stationary causal solution. In our setting, the stationary parameter process $\theta^\circ_k$ has the representation
\begin{equation}\label{eq:nonlinear-system}
    \theta_k^\circ=H(\mathcal{X}_k)=H(\ldots,x_{k-1},x_k),
\end{equation}
for a measurable function $H$ determined by the recursion \eqref{eq:nonlinear-ar2}.
\section{Geometric moment contraction}
\label{sec:gmc}

Geometric moment contraction (GMC) quantifies stability of a stochastic recursion under a synchronous coupling \citep{wu2005nonlinear,Wu2011Asymptotic}. Two trajectories use the same innovations but start from different states. GMC requires their $L^p$ distance to decrease geometrically. At a constant learning rate, fresh gradient noise prevents pathwise convergence to $\theta^\star$ in general \citep{li2026stochastic,wei2025gaussian}. Instead, GMC proves the existence of a stationary solution to the recursion and controls how fast the recursion forgets its initialization.

Let $Z_k=(\theta_k,y_k)$ and write $Z_k=\Phi_{x_k}(Z_{k-1})$, where
\begin{equation*}
    \Phi_x(\theta,y)=\begin{pmatrix}\theta+\beta y-\gamma g(\theta,x)\\\beta y-\gamma g(\theta,x)\end{pmatrix}
\end{equation*}
denotes the one-step transition map in accordance with \eqref{eq:sgdm}. Let $Z_k^z$ denote the recursion started at $z$ and driven by the common innovation sequence.

\begin{definition}[State-space GMC]
\label{def:gmc}
Let $\dist$ be a metric on $\mathbb R^{2d}$.
The SGDM recursion is $L^p$ geometrically moment contracting with respect to $\dist$ if there exists $r\in(0,1)$ such that
\begin{equation*}
\|\dist(\Phi_{x_0}(z), \Phi_{x_0}(z^{\prime}))\|_p \leq r \dist(z, z^{\prime}),
\end{equation*}
for all deterministic $z,z'\in\mathbb R^{2d}$.
\end{definition}

Suppose $L^p$ GMC holds for some $p>1$ with contraction factor $r$ in a norm-induced distance, and $D_p=\|\dist(\Phi_{x_0}(z^\star),z^\star)\|_p<\infty$ where $z^\star=(\theta^\star,0)$.
The iterated-random-function theorem of \citet[Theorem~2]{Wu2004Limit} gives a unique stationary causal solution with a finite $p$th moment. At time $0$, it is given by the $L^p$ limit
\[
    Z_0^\circ = \lim_{m\to\infty} \Phi_{x_0}\circ\Phi_{x_{-1}}\circ\cdots\circ \Phi_{x_{1-m}}(z^\star).
\]
This backward construction fixes the observation time and moves the initialization $m$ steps into the remote past. The limit $Z_0^\circ$ is $\mathcal{F}_0^\circ$-measurable, so that its time shifts define a causal stationary process $(Z_k^\circ)$. The contraction also implies uniqueness and geometric forgetting of every deterministic initialization under a synchronous coupling with $(Z_k^\circ)$. This coupling is used in Section~4 to transfer limit results for the stationary process to trajectories started from deterministic states; see Appendix~\ref{app:stationary-construction}.

The following stochastic-gradient conditions allow us to establish GMC.
\begin{assumption}[Global stochastic-gradient conditions]
\label{ass:global}
For the moment order $p$, the following conditions hold.
\begin{enumerate}[label=(A\arabic*), ref=(A\arabic*)]
    \item \label{ass:strong-monotonicity} There exists $\mu>0$ such that, for all $\theta,\theta'\in\R^d$, $\langle G(\theta)-G(\theta'),\theta-\theta'\rangle\ge\mu|\theta-\theta'|^2$.
    \item \label{ass:unbiased-gradient} $\mathbb E[g(\theta,x_0)]=G(\theta)$ for all $\theta$.
    \item \label{ass:lp-lipschitz-gradient} There is $L_p<\infty$ such that for all $\theta,\theta' \in \R^d$, $\|g(\theta,x_0)-g(\theta',x_0)\|_p\le L_p|\theta-\theta'|$
    \item \label{ass:target-gradient-moment} The stochastic gradient at $\theta^\star$ has a finite $p$th moment: $M_p:=\|g(\theta^\star,x_0)\|_p<\infty$.
\end{enumerate}
\end{assumption}

Assumption~\ref{ass:strong-monotonicity} asserts that $F$ is $\mu$-strongly convex. Assumption~\ref{ass:unbiased-gradient} asserts that we can interchange the expectation in $x$ and gradient in $\theta$ in defining $G(\theta)$. Assumption~\ref{ass:lp-lipschitz-gradient} implies that $G$ is $L_p$-Lipschitz by Jensen's inequality. Assumption~\ref{ass:target-gradient-moment} implies that $g(\theta, x_0)$ is $L^p$ integrable at every $\theta$ by the triangle inequality and \ref{ass:lp-lipschitz-gradient}. 

\subsection{\texorpdfstring{$L^p$ GMC for $p>1$}{Lp GMC for p greater than 1}}

Let $q_{\eta,p}$ be the ordinary-SGD contraction coefficient defined by
{
\begin{equation}\label{eq:sgd-contraction-coefficient}
    q_{\eta,p}^p = \begin{cases}
        1-p\mu\eta+2^{2-p}L_p^p\eta^p, &1<p<2,\\
        (1+\eta L_p)^p-p\eta L_p-p\mu\eta, &p\ge2.
    \end{cases}
\end{equation}
}
For $T_{\eta, x_0}(\theta) = \theta - \eta g(\theta, x_0)$ the ordinary SGD transition map, \citet{li2026stochastic} prove that $\|T_{\eta, x_0}(\theta) - T_{\eta, x_0}(\theta')\|_p \leq q_{\eta,p}|\theta - \theta'|$.
For momentum, the transformed position $u_k=\theta_k+\frac{\beta}{1-\beta}y_k$ satisfies
\[
    u_k=u_{k-1}-\eta g(\theta_{k-1},x_k).
\]
This has the form of an SGD update, but the gradient is evaluated at $\theta_{k-1}$ rather than $u_{k-1}$. For two coupled trajectories, the stochastic-gradient Lipschitz condition bounds the error in the ordinary-SGD comparison by their momentum difference. Appendix~\ref{app:gmc-proofs} gives a nonnegative $2\times2$ matrix $M_{\gamma,\beta,p}$ bounding the conditional $L^p$ norms of the transformed position and momentum differences. Write $r=\rho(M_{\gamma,\beta,p})$. When $r<1$, the matrix's left and right Perron eigenvectors provide weights that combine these two bounds into one $L^p$ contraction. Define the transformed position difference $\Delta u=(\theta-\theta')+\frac{\beta}{1-\beta}(y-y')$. The associated distance is the weighted norm
\begin{equation}\label{eq:weighted-norm}
    \dist_{\mathrm{sf},p}(z,z')^p=|\Delta u|^p+w|y-y'|^p,
\end{equation}
where $w>0$ is the weight defined in Appendix~\ref{app:gmc-distance}. This invertible change of coordinates makes $\dist_{\mathrm{sf},p}$ equivalent to the Euclidean norm on $(\theta,y)$. In particular, $|\theta-\theta'|\le c_\theta\dist_{\mathrm{sf},p}(z,z')$ for a finite, parameter-dependent $c_\theta$.

\begin{theorem}[All-$\beta$ $L^p$ geometric moment contraction]
\label{thm:perron-gmc}
\label{thm:all-beta-lp-gmc}
Suppose Assumption~\ref{ass:global} holds for some $p>1$. Fix $0<\beta<1$ and assume $q_{\eta,p}<1$ and
\begin{equation}
    \beta\eta^2L_p^2 <(1-q_{\eta,p})(1-\beta-\beta\eta L_p).
    \label{eq:perron-condition}
\end{equation}
Then $r<1$ and, for all deterministic $z,z'$,
\begin{equation}
    \|\dist_{\mathrm{sf},p}(\Phi_{x_0}(z),\Phi_{x_0}(z'))\|_p
    \le r\dist_{\mathrm{sf},p}(z,z').
    \label{eq:perron-gmc}
\end{equation}
There is a unique invariant distribution, with a finite $p$th moment, and a unique stationary causal solution $Z_k^\circ$. For every deterministic initialization,
\begin{equation}
\begin{split}
    \|\dist_{\mathrm{sf},p}(Z_k,Z_k^\circ)\|_p
    \le r^k\|\dist_{\mathrm{sf},p}(Z_0,Z_0^\circ)\|_p.
\end{split}
\label{eq:perron-forgetting}
\end{equation}
For every fixed $\beta\in(0,1)$, condition~\eqref{eq:perron-condition} holds for all sufficiently small $\gamma>0$.
\end{theorem}

The theorem allows state-dependent noise and, for $1<p<2$, infinite variance. The admissible step depends on $\beta$. At $\beta=0$, ordinary SGD gives the same stationary conclusion using a separate distance. For fixed $\beta$, $r=1-\mu\eta+O_{\beta,p}(\eta^{p\wedge2})$. The stationary conclusions follow from the discussion after Definition~\ref{def:gmc}, since \ref{ass:target-gradient-moment} ensures $D_p<\infty$.

\subsection{An explicit second-moment bound}

At $p=2$, the weight in \eqref{eq:weighted-norm} is $w=\beta/(1-\beta)^2$, so write $\dist_\beta=\dist_{\mathrm{sf},2}$. Expanding its square directly gives the following alternative sufficient condition. 

\begin{corollary}[Explicit $L^2$ GMC condition]
\label{thm:all-beta-gmc}
\label{cor:all-beta-l2-gmc}
Suppose Assumption~\ref{ass:global} holds with $p=2$. For any fixed $0<\beta<1$, assume
\begin{equation}
    0<\gamma<\gamma_{\rm MS}(\beta)
    :=\frac{2\mu(1-\beta)^2}{L_2^2(1+\beta)}.
    \label{eq:all-beta-step}
\end{equation}
Then the one-step bound~\eqref{eq:perron-gmc} and forgetting bound~\eqref{eq:perron-forgetting} hold with $\dist_\beta=\dist_{\mathrm{sf},2}$ and some contraction factor $r_{\rm MS}(\gamma,\beta)\in(0,1)$.
\end{corollary}

Appendix~\ref{app:stability-comparision} compares the stability regions, and Appendix~\ref{app:stationary-construction} constructs the stationary law.

\section{Averaged SGDM and limit theory}
\label{sec:averaging-fclt}

Fix $(\gamma,\beta)$ and assume $L^p$ GMC with $D_p<\infty$ for some $p>1$. Let $Z_k^\circ=(\theta_k^\circ,y_k^\circ)$ have the unique stationary law. Define its mean and the empirical average of a process started from an arbitrary deterministic state 
\[
\theta_\infty^\star(\gamma, \beta) = \E[\theta_0^\circ],\qquad
\overline\theta_n = \frac{1}{n} \sum_{k=1}^n \theta_k.
\]
The results in this section concern stochastic fluctuations of the average around the stationary mean.
At a fixed learning rate, this mean $\theta^\star_\infty(\gamma, \beta)$ can differ from the minimizer $\theta^\star$, as already occurs for nonlinear SGD \citep{dieuleveut2020bridging}. The averaged CLT therefore describes estimation of the stationary mean.

\subsection{Physical dependence inherited from GMC}

Following the nonlinear autoregressive framework of \citet{chen2016stability}, we quantify temporal dependence through the effect of a local perturbation in the innovation sequence.
Following \eqref{eq:nonlinear-system}, let $x_0'$ be an independent copy of $x_0$, independent of the full innovation sequence. Let $\mathcal{X}_k^{(0)} = (\dots x_{-1}, x_0', x_1, \dots, x_k)$ denote the perturbed innovation sequence, and define the physical dependence coefficient
\begin{equation*}
    \delta_{k,p}=\|\theta_k^\circ-H(\mathcal{X}_k^{(0)})\|_p, \qquad k\ge0.
\end{equation*}
This coefficient measures the effect of the innovation at time \(0\) on the parameter iterate \(k\) steps later.

\begin{proposition}[Functional dependence measure]
\label{prop:physical-dependence}
Under Theorem~\ref{thm:perron-gmc} for $p>1$, or Corollary~\ref{thm:all-beta-gmc} or Proposition~\ref{thm:curvature-rate} for $p=2$, write $\dist$ for the corresponding norm distance and $r<1$ for its contraction factor. Then, with $C_p\le2c_\theta D_p/(1-r)$, where $D_p=\|\dist(\Phi_{x_0}(z^\star),z^\star)\|_p$ and $z^\star=(\theta^\star,0)$,
\begin{equation}
    \delta_{k,p}\le C_p r^k, \qquad k\ge0. \label{eq:exponential-physical-dependence}
\end{equation}
\end{proposition}

Here $c_\theta$ is the finite projection constant satisfying $|\theta-\theta'|\le c_\theta \dist((\theta,y),(\theta',y'))$; its finiteness follows from equivalence of the state-space norm with the Euclidean norm.
Hence the tail sum of the functional dependence coefficients satisfies
$\Theta_{m,p}:=\sum_{k=m}^\infty\delta_{k,p}\le\frac{C_p r^m}{1-r}$.
For \(U_k=\theta_k^\circ-\theta_\infty^\star(\gamma,\beta)\), Wu's projection inequality gives the bound $\|\mathcal P_0U_k\|_p\le \delta_{k,p}$ \citep[Theorem~1(ii)]{wu2005nonlinear}.
This gives summable projective dependence. The resulting partial-sum moment and almost-sure bounds follow from \citet{Wu2007Strong}, while for \(p=2\) we prove the invariance principle using a martingale approximation and the vector martingale FCLT.

\subsection{Averaging rates}

Summable dependence gives bounds on the sampling fluctuations of the average. The resulting rates match those for ordinary constant-step SGD \citep{li2026stochastic}.

\begin{theorem}[Averaged SGDM]
\label{thm:averaged-sgdm}
Assume $L^p$ GMC for some $p>1$ and $D_p<\infty$. For every deterministic initial state $z_0=(\theta_0,y_0)\in\R^{2d}$,
\begin{equation*}
    \|\bar\theta_n-\theta_\infty^\star(\gamma,\beta)\|_p=O\left(n^{1/(p\wedge2)-1}\right).
\end{equation*}
Moreover,
\begin{equation*}
    \bar\theta_n-\theta_\infty^\star(\gamma,\beta)=
    \begin{cases}
        o_{\rm a.s.}(n^{1/p-1}), & 1<p<2,\\
        O_{\rm a.s.}\left(\sqrt{\frac{\log\log n}{n}}\right), & p\ge2.
    \end{cases}
\end{equation*}
\end{theorem}
With fixed $(\gamma,\beta)$, averaging recovers the stationary mean $\theta_\infty^\star(\gamma,\beta)$ even while individual iterates continue to fluctuate. The almost-sure rates describe convergence along a single SGDM trajectory. In particular, a finite moment of order $1<p<2$ suffices, so strong consistency of the average does not require finite variance.

\subsection{Invariance principle and quenched CLT}

An invariance principle describes the joint fluctuations of partial sums over time and yields distributional limits for functionals such as their maximum norm. Functional and conditional limit theory for dependent sequences is developed in \citet{Dedecker2003Conditional,Wu2011Asymptotic}.

Suppose now that $p=2$ and define the long-run covariance
\begin{equation*}
    \Sigma_{\gamma,\beta}^{\rm av}=\sum_{h\in\mathbb{Z}}\Cov(\theta_0^\circ,\theta_h^\circ).
\end{equation*}
The series is absolutely convergent entrywise under \eqref{eq:exponential-physical-dependence}.

\begin{theorem}[Invariance principle]
\label{thm:fclt}
Let $B$ be a standard $d$-dimensional Brownian motion. Assume $L^2$ GMC and $\mathbb{E}|\theta_0^\circ|^2<\infty$. For every deterministic initialization,
\begin{equation}\label{eq:invariance_principle}
    \left(\frac1{\sqrt n}\sum_{k=1}^{\lfloor ns\rfloor}\left(\theta_k-\theta_\infty^\star(\gamma,\beta)\right):0\le s\le1\right)\Rightarrow\left((\Sigma_{\gamma,\beta}^{\rm av})^{1/2}B(s):0\le s\le1\right)
\end{equation}
in $D([0,1],\R^d)$ as $n \rightarrow \infty$.
\end{theorem}

Projecting the limit \eqref{eq:invariance_principle} to $s=1$ gives the central limit theorem for averaged SGDM.

\begin{corollary}[Quenched CLT for averaged SGDM]
\label{coro:average-clt}
Under the conditions of Theorem~\ref{thm:fclt}, from every deterministic initial parameter iterate and momentum,
\begin{equation*}
    \sqrt{n}\left(\bar\theta_n-\theta_\infty^\star(\gamma,\beta)\right)\Rightarrow N(0,\Sigma_{\gamma,\beta}^{\rm av}).
\end{equation*}
\end{corollary}

The limit is quenched with respect to initialization because the GMC coupling error is $O_{L^2}(1)=o_{L^2}(\sqrt n)$, uniformly over the partial sums; see \citet[Theorem~3.2]{li2026stochastic} for the analogous quenched CLT for ordinary SGD.

\section{Last-iterate limit theory}
\label{sec:last-iterate}

The preceding section fixes $(\gamma,\beta)$ and sends the averaging horizon to infinity.  We now let $\gamma\to0$ and $n=n(\gamma)\to\infty$ jointly, and study the nonaveraged last iterate $\theta_n$.  Throughout this section, $\beta\in[0,1)$ is fixed, $\theta_0$ is deterministic, and $y_0=0$.  Write $e_k=\theta_k-\theta^\star$ and $\eta=\gamma/(1-\beta)$.

We impose local second-order regularity only for the last-iterate result.

\begin{assumption}[Local Hessian smoothness]
\label{ass:local-hessian}
The map $G$ is differentiable at $\theta^\star$, and $A:=DG(\theta^\star)=\nabla^2 F(\theta^\star)$ is symmetric. Moreover, there exists $\kappa_R>0$ and $L_A<\infty$ such that, whenever $|\theta-\theta^\star|\le\kappa_R$,
\[ |G(\theta)-A(\theta-\theta^\star)|\le L_A|\theta-\theta^\star|^2. \]
\end{assumption}

Under Assumptions~\ref{ass:global} and~\ref{ass:local-hessian}, the eigenvalues of $A$ lie in $[\mu,L_2]$. Noting that $\mathbb{E}[g(\theta^\star,x_k)]=G(\theta^\star)=0$, we define $\varepsilon_k^\star=g(\theta^\star,x_k)$ and $S=\mathbb{E}[\varepsilon_0^\star(\varepsilon_0^\star)^\top]$. Further, let
\[
    \Gamma=\int_0^\infty e^{-As}Se^{-As}\,ds.
\]
The matrix $\Gamma$ is the unique symmetric positive-semidefinite solution of $A\Gamma+\Gamma A=S$.

Lemma~\ref{lem:last-iterate-auxiliary} in Appendix~\ref{app:last-iterate-proofs} shows that, under the burn-in condition below and for fixed $\beta$, $\|\theta_n-\theta^\star\|_2=O(\sqrt\gamma)$. The proof linearizes the mean gradient by replacing $G(\theta)$ with $A(\theta-\theta^\star)$ and approximates the state-dependent stochastic-gradient noise by its value at the minimizer, replacing $g(\theta,x_k)-G(\theta)$ with $\varepsilon_k^\star$. The resulting linear SGDM recursion approximates the original recursion with an error negligible after division by $\sqrt\gamma$, and yields the Gaussian limit.

\begin{theorem}[Quenched last-iterate CLT]
\label{thm:last-iterate-clt}
Suppose Assumptions~\ref{ass:global} and~\ref{ass:local-hessian} hold with $p=2$.  Let $\gamma\to0$ with fixed $\beta\in[0,1)$, $\theta_0 \in \R^d$, and $y_0=0$.  If $n$ satisfies
\begin{equation}
    n\gamma\to\infty,\quad n\sqrt\gamma(1-\mu\gamma/(1-\beta))^{n-1}\to0,
\label{eq:last-iterate-burnin}
\end{equation}
then
\begin{equation}
    \frac{\theta_n-\theta^\star}{\sqrt\gamma}\ \Rightarrow\ N\!\left(0,\frac{\Gamma}{1-\beta}\right).
    \label{eq:last-iterate-clt}
\end{equation}
\end{theorem}

A sufficient condition for the burn-in requirements in~\eqref{eq:last-iterate-burnin} is $n\gamma\ge\nu\log(1/\gamma)$ with
$\nu>\frac{1-\beta}{2\mu}$ for every sufficiently small $\gamma$. On the $\sqrt\eta$ scale, $(\theta_n - \theta^\star) / \sqrt\eta \Rightarrow N\!\left(0,\Gamma\right)$ holds. For fixed $\beta$, the leading covariance at the $\sqrt\eta$ scale therefore agrees with that of ordinary SGD.

\section{Forgetting and averaged fluctuations under stochastic curvature}
\label{sec:quadratic-acceleration}

In a random-curvature quadratic model, we relate initialization forgetting to the averaged limiting covariance $\Sigma_{\gamma,\beta}^{\rm av}$ and bound finite-average mean-square error. We make the dependence on the condition number explicit for the accelerated tuning studied below.

\subsection{Contraction and limiting covariance}
\label{sec:stochastic-curvature}

Consider the quadratic model
\begin{equation}
\begin{aligned}
 g(\theta,x_k)&=H_k(\theta-\theta^\star)+\varepsilon_k,\\
 H_k&=Q\operatorname{diag}(H_{1,k},\ldots,H_{d,k})Q^\top,
\end{aligned}\label{eq:commuting-curvature-model}
\end{equation}
where $Q$ is a fixed orthogonal matrix and the pairs $(H_k,\varepsilon_k)$ are i.i.d. Write $\mathbb EH_{i,k}=\lambda_i\in[\mu,L]$, where $0<\mu\le L$, and $\operatorname{Var}(H_{i,k})=v_i<\infty$. Assume $\mathbb E\varepsilon_k=0$ and $\mathbb E|\varepsilon_k|^2<\infty$. Coordinates of $H_k$ may be dependent, and curvature and additive noise need not be independent. These conditions imply Assumption~\ref{ass:global} with $p=2$ and $L_2=(\mathbb E|H_k|^2)^{1/2}$. Here $L$ bounds the eigenvalues of the mean curvature and differs from $L_2$. Appendix~\ref{app:coordinate-batching} gives a coordinate-sampled least-squares example.

To measure forgetting of the initial state, compare two state trajectories $Z_k^z$ and $Z_k^{z'}$ driven by the same observations but started from different deterministic states $z, z'$. Let $q$ be the infimum of factors $r>0$ for which
\begin{equation*}
    \|Z_k^z-Z_k^{z'}\|_2\le C_r r^k|z-z'|
\end{equation*}
holds with a finite constant $C_r$ independent of $k,z,z'$. We call $q$ the forgetting factor and $1-q$ the forgetting rate. Thus $q<1$ means that the initial state is forgotten geometrically, and a smaller $q$ means faster forgetting.

The following proposition connects this forgetting factor to the GMC framework using standard mean-square stability arguments \citep{boyd1994linear,andreyev2026momentum}. Appendix~\ref{app:curvature-rate} computes $q$ from the second-moment recursion and gives scalar criteria for stability and prescribed forgetting factors.

\begin{proposition}[Mean-square contraction and stationarity]
\label{thm:curvature-rate}
In model~\eqref{eq:commuting-curvature-model}, let $0\le\beta<1$ and $\gamma>0$.
If $q<1$, then for every $r\in(q,1)$ there is a distance $\dist_r$, induced by a quadratic norm equivalent to the Euclidean norm, such that
{
\begin{equation*}
    \|\dist_r(\Phi_{x_0}(z),\Phi_{x_0}(z'))\|_2\le r\,\dist_r(z,z')
\end{equation*}
}
for all deterministic $z,z'\in\mathbb R^{2d}$. There is a unique invariant probability law with finite second moment and a causal stationary realization.
\end{proposition}

The averaged CLT of Section~\ref{sec:averaging-fclt} applies whenever $q<1$. The next proposition gives its center and the stationary and averaged variances.

\begin{proposition}[Stationary variance and averaged limiting variance]
\label{prop:curvature-variance}
In model~\eqref{eq:commuting-curvature-model}, suppose $q<1$. The stationary mean is $\theta^\star$. For a coordinate of $Q^\top(\theta_k-\theta^\star)$ write $h=\lambda_i$, $v=v_i$, and $\sigma^2=\mathbb E(Q^\top\varepsilon_k)_i^2$. Its stationary variance $V$ and averaged CLT variance $\Sigma_i^{\rm av}$ are
\begin{equation}
 \begin{aligned}
 V&=\frac{\gamma\sigma^2(1+\beta)}
 {2h(1-\beta^2)-\gamma((1-\beta)h^2+(1+\beta)v)},\qquad
 \Sigma^{\rm av}_{i}=\frac{\sigma^2+vV}{h^2}.
 \end{aligned}
 \label{eq:curvature-variances}
\end{equation}
Here $\Sigma^{\rm av}_{i}=(Q^\top\Sigma_{\gamma,\beta}^{\rm av}Q)_{ii}$, and the CLT holds from every deterministic initialization. These formulas require only the stated second moments.
\end{proposition}

Write $V_i$ and $\sigma_i^2$ for the stationary and additive-noise variances in coordinate $i$. With fixed curvature, the averaged limiting variance is independent of stable tuning. With random curvature, it depends on tuning through $V_i$. The multiplicative-noise correction also appears in linear stochastic approximation \citep[Theorem~2]{Mou2020On}. Our formulas hold throughout $q<1$. Proofs are in Appendices~\ref{app:curvature-rate} and~\ref{app:curvature-variance}.

\subsection{Forgetting rates and limiting variance}
\label{sec:curvature-variance}

Faster forgetting reduces initialization effects but need not reduce averaged fluctuations. For a prescribed forgetting rate, we bound the variance relative to $\sigma_i^2/\lambda_i^2$, its value with curvature fixed at its mean and unchanged additive noise. We also give an accelerated tuning under small curvature variance \citep{polyak1964some,can2019accelerated}.

\begin{theorem}[Forgetting rates and averaged limiting variance]
\label{thm:curvature-acceleration}
\label{thm:rate-variance-cost}
In model~\eqref{eq:commuting-curvature-model}, suppose $\kappa=L/\mu\ge4$ and let $q$ be the forgetting factor defined above.

\textup{(i)} Suppose $q<1$, $\mu,L$ are actual mean eigenvalues, and a coordinate with $\lambda_i=L$ has curvature variance $v_L=\nu L^2>0$. Write $\sigma_L^2$ and $\Sigma_L^{\rm av}$ for its additive-noise and averaged limiting variances. Every stable tuning satisfies
\begin{equation}
 1-q\le\frac4{\nu\kappa+2}.
 \label{eq:curvature-noise-necessity}
\end{equation}
If also $\sigma_L^2>0$, then
\begin{equation}
 \frac{\Sigma_L^{\rm av}}{\sigma_L^2/L^2}
 \ge1+\frac{\nu\kappa}{2}\frac{1-q}{1+q}.
 \label{eq:rate-variance-lower}
\end{equation}
In particular, $q\le1-a/\sqrt\kappa$ with a constant $a>0$ implies
\begin{equation*}
 \frac{\Sigma_L^{\rm av}}{\sigma_L^2/L^2}\ge1+\frac a4\nu\sqrt\kappa.
\end{equation*}

\textup{(ii)} Conversely, suppose $v_i\le\nu\lambda_i^2$ for every coordinate and assume $\nu\le \kappa^{-1/2}/16$. At
\begin{equation}
 \gamma_{\rm SGDM}=L^{-1},\qquad \beta_{\rm SGDM}=(1-\kappa^{-1/2})^2,
 \label{eq:curvature-retuning}
\end{equation}
we have $1-\kappa^{-1/2}\le q<\bar q:=1-\kappa^{-1/2}/2$ and, whenever $\sigma_i^2>0$,
\begin{equation}
 1\le\frac{\Sigma_i^{\rm av}}{\sigma_i^2/\lambda_i^2}\le1+\nu\sqrt\kappa.
 \label{eq:rate-variance-upper}
\end{equation}
\end{theorem}

Under part~\textup{(i)}, accelerated forgetting is impossible when $\nu\sqrt\kappa$ grows without bound. When both parts apply and $\sigma_L^2>0$, the relative variance increase over the fixed-curvature reference is of order $\nu\sqrt\kappa$ in the $L$ direction at the tuning in part~\textup{(ii)}. Thus accelerated forgetting can coexist with a nonvanishing variance increase. When $\nu\sqrt\kappa$ tends to zero, this tuning also makes the variance ratio tend to one in every coordinate with positive additive-noise variance. Appendix~\ref{app:rate-variance-proof} proves the theorem.

In the coordinate-sampled least-squares example of Appendix~\ref{app:coordinate-batching}, averaging $B$ iid samples per iteration gives $\nu=(d-1)/B$. Thus $B\ge16(d-1)\sqrt\kappa$ ensures the curvature-variance condition in part~\textup{(ii)}. Since each iteration uses $B$ sample gradients, this acceleration in iteration count alone does not establish a reduction in total gradient evaluations.

\subsection{Error of the averaged SGDM estimator}
\label{sec:quadratic-finite-averaging}

We bound finite-average mean-square error for stationary and deterministic initializations using a coupling bound with a constant independent of $\kappa$.

For integers $b\ge0$ and $n\ge1$, let
\begin{equation*}
 \bar\theta_{b,n}=\frac1n\sum_{k=b+1}^{b+n}\theta_k,
 \qquad \mathcal S=\operatorname{tr}\Sigma_{\gamma,\beta}^{\rm av}=\sum_i\Sigma_i^{\rm av}.
\end{equation*}
Here $b$ is the burn-in length. For deterministic $\theta_0,y_0$, define the weighted mean-square initial difference from a stationary trajectory driven by the same observations, including its stationary variances, by
{
\begin{equation*}
 \mathcal D_0=|\theta_0-\theta^\star|^2+\sum_i\frac{L}{\lambda_i}\bigl((Q^\top y_0)_i\bigr)^2+\left(1+\frac2{1+\beta}\right)\sum_iV_i.
\end{equation*}
}

\begin{theorem}[Mean-square error of the averaged SGDM estimator]
\label{thm:rv-finite-horizon}
Under the assumptions and tuning in part~\textup{(ii)} of Theorem~\ref{thm:curvature-acceleration}, there is a universal constant $C$ such that a stationary initialization satisfies
\begin{equation}
 \left|\mathbb E|\bar\theta_{b,n}^\circ-\theta^\star|^2-\frac{\mathcal S}{n}\right|
 \le\frac{C\sqrt\kappa\,\mathcal S}{n^2}.
 \label{eq:rv-stationary-remainder}
\end{equation}
For every deterministic initial state,
\begin{equation}
 \|\bar\theta_{b,n}-\theta^\star\|_2\le\left(\frac{\mathcal S}{n}\left(1+\frac{C\sqrt\kappa}{n}\right)\right)^{1/2}+\frac{C\sqrt{\kappa\mathcal D_0}}{n}\exp\left(-\frac{b+1}{2\sqrt\kappa}\right).
 \label{eq:rv-finite-horizon-main}
\end{equation}
If $y_0=0$, then $\mathcal D_0\le|\theta_0-\theta^\star|^2+3\sqrt\kappa\,\mathcal S$.
\end{theorem}

For $\mathcal S>0$, the stationary mean-square error is asymptotic to $\mathcal S/n$ when $n/\sqrt\kappa\to\infty$. The second bound controls initialization and burn-in. Appendix~\ref{app:rv-finite-horizon} gives explicit constants and sufficient averaging and burn-in lengths for a prescribed relative accuracy.

\section{Numerical experiments}
\label{sec:numerical}

Figure~\ref{fig:numerical-report} summarizes stationary behavior, averaging limits, quadratic burn-in, and the rate--variance tradeoff; Appendix~\ref{app:numerical} gives supplementary diagnostics.

\subsection[Stationarity and averaging limits]{Stationarity and averaging limits}
\label{sec:numerical-stationarity-averaging}

We use the scalar nonlinear model
$g(\theta,(R,\varepsilon))=(1+R)\theta+\frac14\sin\theta+\varepsilon$,
where $\varepsilon\sim N(0,1)$. The pairs $(R_k,\varepsilon_k)$ are iid, with independent components. Then $G(\theta)=\theta+\sin\theta/4$ and $\theta^\star=0$; throughout, $y_0=0$. Here $T_\alpha$ denotes a Student-$t$ random variable with $\alpha$ degrees of freedom.

Panel (a) takes $R=0.15T_{1.5}$ and $(\beta,\gamma)=(0.5,0.01)$, so the curvature has a finite moment of order $1.3$ but infinite variance. Across $\theta_0=1,2,4$, the marginal quantiles of $|\theta_k-\theta^\star|$ stabilize at similar levels, consistent with initialization forgetting toward a common stationary regime in Theorem~\ref{thm:perron-gmc}. Figure~\ref{fig:numerical-appendix}(a) and Table~\ref{tab:gmc} give the fractional-moment coupling diagnostic and verification of the sufficient GMC conditions.

Panel (b) uses $R=0.5T_3$ and $(\beta,\gamma)=(0.5,0.05)$, in the finite-second-moment regime. For $0\le s\le1$, define
$\mathbb S_n(s)=\sum_{k=1}^{\lfloor ns\rfloor}(\theta_k-\theta_\infty^\star(\gamma,\beta))/\sqrt{n\hat\Sigma^{\rm av}}$,
using a common reference variance estimate $\widehat\Sigma^{\rm av}$. With $M_n=\sup_{0\le s\le1}|\mathbb S_n(s)|$, we compare its law with that of $\sup_{0\le s\le1}|B(s)|$ for standard Brownian motion $B$. The empirical Kolmogorov distances decrease to small values for all three initial values $\theta_0=-3,1,3$, illustrating Theorem~\ref{thm:fclt}. At the largest $n$, the distances are below the empirical-CDF sampling bound, so their ordering is unresolved; Appendix~\ref{app:numerical-invariance} gives the sampling bounds and reference-variance calculation.

\begin{figure}[H]
    \centering
    \includegraphics[width=\textwidth]{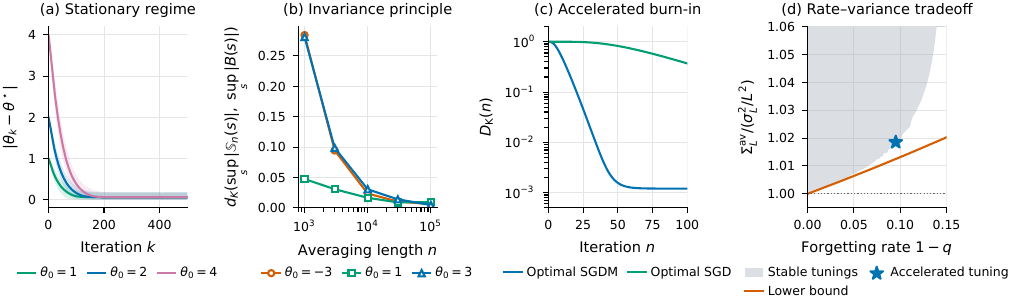}
    \caption{(a) Stationary behavior: marginal median (lines) and 5--95\% quantiles (bands) of $|\theta_k-\theta^\star|$ under nonlinear heavy-tailed SGDM. (b) Kolmogorov distance between the normalized partial-sum maximum and the Brownian maximum. (c) Kolmogorov distance under optimal quadratic tuning. (d) Averaged limiting variance versus forgetting rate under stochastic curvature, with the theoretical lower bound and explicit tuning. The gray region is drawn by interpolating the numerical lower envelope.}
    \label{fig:numerical-report}
\end{figure}

\subsection[Quadratic burn-in and rate--variance tradeoff]{Quadratic burn-in and rate--variance tradeoff}
\label{sec:numerical-contraction}

Panel (c) considers the additive-noise quadratic with $A=\diag(1,100)$, comparing SGDM and SGD at their optimal tunings for quadratic forgetting. For each method, $D_{\rm K}(n)$ measures the first-coordinate Kolmogorov distance to its small-step Gaussian reference using its own normalized learning rate (Appendix~\ref{app:numerical-quadratic}). SGDM removes the slow-coordinate initialization transient faster, consistent with the quadratic forgetting rates. The small-step limit of Theorem~\ref{thm:last-iterate-clt} is illustrated separately in Figure~\ref{fig:numerical-appendix}(d) and Table~\ref{tab:student_t}, using decreasing $\gamma$ at fixed $\beta=0.5$ with burn-in satisfying~\eqref{eq:last-iterate-burnin}.

Panel (d) examines forgetting and averaged fluctuations in model~\eqref{eq:commuting-curvature-model}, with $(\mu,L)=(0.01,1)$ and relative curvature variance $\nu=0.005$. We plot $\Sigma_L^{\rm av}/(\sigma_L^2/L^2)$ against $1-q$ over stable tunings. The lower bound~\eqref{eq:rate-variance-lower} increases with the required forgetting rate; $\sigma_L^2/L^2$ is the fixed-curvature reference with the same additive noise, not the SGD variance under model~\eqref{eq:commuting-curvature-model}. The star marks the explicit accelerated tuning~\eqref{eq:curvature-retuning}. The gray lower edge interpolates a numerical grid envelope and does not certify attainability. Appendix~\ref{app:numerical-curvature} checks the variance formulas by Monte Carlo.

\section{Conclusion}
We developed a nonlinear time-series analysis of constant-step SGDM based on geometric moment contraction. Under second moments, this contraction yields a central limit theorem for averaged SGDM. At a fixed step size, averaging estimates the stationary mean, which can differ from the minimizer. In quadratics with commuting random curvature, faster forgetting need not reduce averaged fluctuations. Our bounds specify when the variance increase from accelerated tuning becomes negligible relative to the variance with fixed curvature.

Our last-iterate Gaussian limit keeps the momentum coefficient fixed as the step size decreases. Future work could allow the coefficient to approach one at the same time, and extend the comparison between forgetting rates and averaged variance to nonquadratic objectives.

\bibliographystyle{plainnat}
\bibliography{references.bib}

\newpage
\appendix
\section*{Appendix}
\section{Proofs for geometric moment contraction}
\label{app:gmc-proofs}

\subsection{Proof of Theorem~\ref{thm:perron-gmc}}
\label{app:gmc-distance}

\begin{proof}

Let
{
\[
    \Delta\theta_k=\theta_k-\theta_k',\quad \Delta y_k=y_k-y_k',\quad
    \Delta u_k=\Delta\theta_k+\frac{\beta}{1-\beta}\Delta y_k,\quad
    h_k=g(\theta_{k-1},x_k)-g(\theta_{k-1}',x_k).
\]
}
By~\eqref{eq:sgdm},
{
\[
    \begin{pmatrix}\Delta u_k\\\Delta y_k\end{pmatrix}
    =\begin{pmatrix}I_d&0\\0&\beta I_d\end{pmatrix}
      \begin{pmatrix}\Delta u_{k-1}\\\Delta y_{k-1}\end{pmatrix}
      -\begin{pmatrix}\eta I_d\\\gamma I_d\end{pmatrix}h_k.
\]
}
The ordinary-SGD coefficient $q_{\eta,p}$ is defined in~\eqref{eq:sgd-contraction-coefficient}.

Let $\Delta u=\Delta u_{k-1}$, $v=\Delta y_{k-1}$, and
{
\[
    \hat h=g(\theta'_{k-1}+\Delta u,x_k)-g(\theta'_{k-1},x_k).
\]
}
By the ordinary-SGD bound,
{
\[
    \big(\mathbb E[|\Delta u-\eta\hat h|^p\mid\mathcal F_{k-1}]\big)^{1/p}
    \le q_{\eta,p}|\Delta u|.
\]
}
Since $\Delta\theta_{k-1}=\Delta u-\beta v/(1-\beta)$, Assumption~\ref{ass:lp-lipschitz-gradient} gives
{
\[
    \big(\mathbb E[|h_k-\hat h|^p\mid\mathcal F_{k-1}]\big)^{1/p}
    \le\frac{\beta L_p}{1-\beta}|v|,
    \qquad
    \big(\mathbb E[|h_k|^p\mid\mathcal F_{k-1}]\big)^{1/p}
    \le L_p\left(|\Delta u|+\frac{\beta}{1-\beta}|v|\right).
\]
}
Conditional on $\Fcal_{k-1}$, Minkowski therefore gives
{
\begin{equation}\
\begin{pmatrix}
(\mathbb E[|\Delta u_k|^p\mid\mathcal F_{k-1}])^{1/p}\\
(\mathbb E[|\Delta y_k|^p\mid\mathcal F_{k-1}])^{1/p}
\end{pmatrix}
\le \underbrace{\begin{pmatrix}q_{\eta,p}&\dfrac{\beta\eta L_p}{1-\beta}\\[2mm](1-\beta)\eta L_p&\beta(1+\eta L_p)\end{pmatrix}}_{M_{\gamma,\beta,p}}\binom{|\Delta u|}{|v|}.
\label{eq:slow-fast-component-bounds}
\end{equation}
}
\label{eq:perron-comparison}
Write $q=q_{\eta,p}$, $\alpha=\frac{\beta\eta L_p}{1-\beta}$, $\zeta=(1-\beta)\eta L_p$, and $ s=\beta(1+\eta L_p)$, such that $M=\begin{pmatrix}q&\alpha\\\zeta&s\end{pmatrix}$.
Its Perron root is
\[
    r=\frac{q+s+\sqrt{(q-s)^2+4\alpha\zeta}}{2}.
\]
Condition~\eqref{eq:perron-condition} assumes $q<1$ and $0<\alpha\zeta<(1-q)(1-s)$, which imply $s<1$ and $\det(I_2-M)=(1-q)(1-s)-\alpha\zeta>0$, hence $q<r<1$.

For the Perron eigenvectors we choose $b=(\alpha,r-q)^\top$ and $\ell=(\zeta,r-q)^\top$, such that
{
\[
    Mb=rb,\qquad \ell^\top M=r\ell^\top,\qquad
    \pi_{ij}=\frac{M_{ij}b_j}{rb_i}\ge0,\qquad \sum_j\pi_{ij}=1.
\]
}
Here $i,j\in\{1,2\}$, with coordinate $1$ corresponding to $\Delta u$ and coordinate $2$ to $\Delta y$.
For $a\in[0,\infty)^2$, Jensen's inequality gives
\begin{align*}
    \sum_i\frac{\ell_i}{b_i^{p-1}}(M a)_i^p
    &=r^p\sum_i\ell_i b_i\left(\sum_j\pi_{ij}\frac{a_j}{b_j}\right)^p\\
    &\le r^{p-1}\sum_{i,j}\ell_iM_{ij}\frac{a_j^p}{b_j^{p-1}}
    =r^p\sum_j\frac{\ell_j}{b_j^{p-1}}a_j^p.
\end{align*}
Normalizing the weights gives
\[
    \frac{\ell_2/b_2^{p-1}}{\ell_1/b_1^{p-1}}=
    \frac{\alpha^{p-1}(r-q)^{2-p}}{\zeta}
    =\frac{\beta^{p-1}}{(1-\beta)^p}
      \left(\frac{r-q}{\eta L_p}\right)^{2-p}=w.
\]

As such, we define
\begin{equation}
    \dist_{\mathrm{sf},p}(z,z')^p=\left|\theta-\theta'+\frac{\beta}{1-\beta}(y-y')\right|^p+w|y-y'|^p.
    \label{eq:slow-fast-metric}
\end{equation}

With $a=(|\Delta u|,|v|)^\top$,~\eqref{eq:slow-fast-component-bounds} yields
{
\[
    \mathbb E[|\Delta u_k|^p+w|\Delta y_k|^p\mid\mathcal F_{k-1}]
    \le r^p(|\Delta u|^p+w|v|^p).
\]
}
This proves the one-step contraction. Conditioning and iteration also extend it to any random initial pair independent of the future innovations.

For fixed $\beta$, condition~\eqref{eq:perron-condition} holds for sufficiently small $\eta>0$, since
\begin{align*}
    q_{\eta,p}&=1-\mu\eta+O(\eta^{p\wedge2})<1,\qquad
    \det(I_2-M)=\mu(1-\beta)\eta+O(\eta^{p\wedge2})>0.
\end{align*}

To obtain the leading rate of $r$, the eigenvalue equation $\det(rI_2-M)=(r-q)(r-s)-\alpha\zeta = 0$ gives
$r= q + \frac{\alpha\zeta}{r-s}$.
For fixed $\beta$, the Perron root formula yields $r-s=\frac{q-s+\sqrt{(q-s)^2+4\alpha\zeta}}{2} \to 1-\beta>0$ as $\eta \to 0$. Thus
\[
    r=q+\frac{\beta\eta^2L_p^2}{r-s}=q+O_\beta(\eta^2)=1-\mu\eta+O_{\beta,p}(\eta^{p\wedge2}).
\]

Existence, uniqueness, and forgetting follow from the construction below.
\end{proof}

\subsection{Construction and uniqueness of the stationary process}
\label{app:stationary-construction}

We give the construction under either set of sufficient contraction conditions in Section~\ref{sec:gmc}. Write $\dist$ for the corresponding distance, which is induced by a norm equivalent to the Euclidean norm, and write $r<1$ for its one-step contraction factor. Let $z^\star=(\theta^\star,0)$ and
{
\[
    D_p=\|\dist(\Phi_{x_0}(z^\star),z^\star)\|_p.
\]
}
For the distance~\eqref{eq:slow-fast-metric},
\[
    D_p=(\eta^p+w\gamma^p)^{1/p}M_p<\infty.
\]
This also applies to $\dist=\dist_\beta$ with $p=2$.

Using the two-sided innovations, define $B_0=z^\star$ and
\[
    B_m=\Phi_{x_0}\circ\Phi_{x_{-1}}\circ\cdots\circ\Phi_{x_{1-m}}(z^\star),\qquad m\ge1.
\]
The extra remote innovation in $B_{m+1}$ is independent of the $m$ common subsequent maps. Iterating the conditional contraction gives
{
\[
    \|\dist(B_{m+1},B_m)\|_p\le r^mD_p.
\]
}
The series is summable. The sequence is Cauchy in $L^p$ in this complete equivalent norm, and its increments are absolutely summable almost surely by Tonelli's theorem. Its limit $Z_0^\circ$ is measurable with respect to $\sigma(x_j:j\le0)$. Shifting the construction defines a stationary causal process $Z_k^\circ$. The one-step contraction allows us to pass to the $L^p$ limit in the recursion because $x_{k+1}$ is independent of the approximating states. Hence $Z_{k+1}^\circ=\Phi_{x_{k+1}}(Z_k^\circ)$, and
{
\begin{equation}
    \|\dist(Z_0^\circ,z^\star)\|_p
    \le\frac{D_p}{1-r}.
    \label{eq:perron-stationary-moment}
\end{equation}
}
Coupling a deterministic initialization with $Z_0^\circ$ using the same future innovations proves~\eqref{eq:perron-forgetting}. If another stationary causal solution has a finite $p$th moment, its distance from $B_m$ is at most $r^m$ times its stationary $L^p$ distance from $z^\star$. Letting $m\to\infty$ proves equality with $Z_0^\circ$.

The invariant distribution is unique even among invariant distributions without an assumed moment. To see this, start independent states $Z_0\sim\nu$ and $Z_0^\circ\sim\pi$, where $\nu$ is any invariant distribution and $\pi$ is the law just constructed. Both are independent of future innovations. For any function $f$ bounded by one and Lipschitz with constant one in $\dist$, conditional contraction gives
{
\[
    |\nu f-\pi f|
    \le\mathbb E\big[\min(2,r^k \dist(Z_0,Z_0^\circ))\big]\to0
\]
}
by dominated convergence. Bounded Lipschitz functions determine probability laws in this metric, so $\nu=\pi$. In particular, every invariant law has the finite moment in~\eqref{eq:perron-stationary-moment}.

For completeness, when $\beta=0$ and $q=q_{\gamma,p}<1$, choose
{
\[
    c_0=\frac{1-q}{2\gamma L_p},\qquad
    \dist_0(z,z')=|\theta-\theta'|+c_0|y-y'|.
\]
}
Then
{
\[
    \|\dist_0(\Phi_{x_0}(z),\Phi_{x_0}(z'))\|_p
    \le(q+c_0\gamma L_p)|\theta-\theta'|
    \le\frac{1+q}{2}\dist_0(z,z').
\]
}
The displacement at $z^\star$ has moment $\gamma(1+c_0)M_p$. The same construction therefore covers the augmented state at zero momentum without using a degenerate weight.

\subsection{\texorpdfstring{Proof of Corollary~\ref{thm:all-beta-gmc}}{Proof of the second-moment corollary}}

\begin{proof}
Let $L=L_2$ and let $\mathcal F_{k-1}$ include both coupled initial states and the shared observations through time $k-1$. Condition on $\mathcal{F}_{k-1}$ and let
\[
    \bar h_k=G(\theta_{k-1})-G(\theta_{k-1}').
\]
Then $\Delta\theta_{k-1}$ and $\Delta y_{k-1}$ are fixed, and Assumption~\ref{ass:global} gives
\begin{equation}
    \langle \Delta\theta_{k-1},\bar h_k\rangle\ge\mu|\Delta\theta_{k-1}|^2,
    \qquad
    \mathbb{E}(|h_k|^2\mid\mathcal{F}_{k-1})\le L^2|\Delta\theta_{k-1}|^2,
    \qquad
    |\bar h_k|\le L|\Delta\theta_{k-1}|.
    \label{eq:appendix-sector-bounds}
\end{equation}

Let
\[
    c_\beta=\frac{\beta}{1-\beta}, \qquad a_\beta=\frac{\beta}{(1-\beta)^2}, \qquad P_\beta=\begin{pmatrix}
        1&c_\beta\\
        c_\beta&c_\beta^2+a_\beta
    \end{pmatrix}\succ0.
\]
For $a,b\in\mathbb R^d$, define the block norm by
{
\[
    |(a,b)|_{P_\beta}^2:=\begin{pmatrix}a\\b\end{pmatrix}^{\top}(P_\beta\otimes I_d)\begin{pmatrix}a\\b\end{pmatrix}=|a+c_\beta b|^2+a_\beta|b|^2.
\]
}
Thus $\dist_\beta$ is the distance induced by this block norm:
\[ \dist_\beta((\theta,y),(\theta',y'))^2=|(\theta-\theta',y-y')|_{P_\beta}^2. \]
Expanding the Lyapunov function and taking the conditional expectation gives
{
\begin{align*}
\mathbb{E}\Big[|\Delta u_k|^2+a_\beta|\Delta y_k|^2\mid\mathcal{F}_{k-1}\Big]
&= |\Delta u_{k-1}|^2+a_\beta|\Delta y_{k-1}|^2
    -2\eta\langle \Delta u_{k-1},\bar h_k\rangle
    -2a_\beta\beta\gamma\langle \Delta y_{k-1},\bar h_k\rangle\\
&\quad +(\eta^2+a_\beta\gamma^2) \mathbb{E}(|h_k|^2\mid\mathcal{F}_{k-1})
    -a_\beta(1-\beta^2)|\Delta y_{k-1}|^2\\
&\le |\Delta u_{k-1}|^2+a_\beta|\Delta y_{k-1}|^2
    -\left(\frac{2\mu\gamma}{1-\beta}
        -\frac{1+\beta}{(1-\beta)^2}L^2\gamma^2\right)|\Delta\theta_{k-1}|^2\\
&\quad + 2\frac{\beta(1+\beta)}{(1-\beta)^2}L\gamma|\Delta\theta_{k-1}||\Delta y_{k-1}|
-\frac{\beta(1+\beta)}{1-\beta}|\Delta y_{k-1}|^2,
\end{align*}
}
using bounds from \eqref{eq:appendix-sector-bounds}.

Define
\begin{equation*}
    K_{\gamma,\beta}=\begin{pmatrix}
        \dfrac{2\mu\gamma(1-\beta)-(1+\beta)L^2\gamma^2}{(1-\beta)^2} & -\dfrac{\beta(1+\beta)L\gamma}{(1-\beta)^2}\\[3mm]
        -\dfrac{\beta(1+\beta)L\gamma}{(1-\beta)^2} & \dfrac{\beta(1+\beta)}{1-\beta}
    \end{pmatrix}.
\end{equation*}
Then the preceding bound can be written as
{
\begin{equation}
\begin{aligned}
    \mathbb{E}\left[|(\Delta\theta_k,\Delta y_k)|_{P_\beta}^2\mid\mathcal{F}_{k-1}\right] &\le |(\Delta\theta_{k-1},\Delta y_{k-1})|_{P_\beta}^2
    -\begin{pmatrix}
        |\Delta\theta_{k-1}|\\
        |\Delta y_{k-1}|
    \end{pmatrix}^{\top}K_{\gamma,\beta}\begin{pmatrix}
        |\Delta\theta_{k-1}|\\
        |\Delta y_{k-1}|
    \end{pmatrix}.
\end{aligned}
\label{eq:all-beta-decrement}
\end{equation}
}
Direct algebra gives
\[ \det K_{\gamma,\beta}=\frac{\beta(1+\beta)\gamma}{(1-\beta)^4}\left(2\mu(1-\beta)^2-(1+\beta)L^2\gamma\right). \]
Under the $\gamma$ range from \eqref{eq:all-beta-step}, both the leading principal minor and the determinant are positive, so $K_{\gamma,\beta}\succ0$.
We now bound the decrement in~\eqref{eq:all-beta-decrement} from below.
Let $w_{k-1}=\big(|\Delta\theta_{k-1}|,|\Delta y_{k-1}|\big)^\top$.
Then
{
\begin{align*}
    w_{k-1}^\top P_\beta w_{k-1}
    &=|\Delta\theta_{k-1}|^2+2c_\beta|\Delta\theta_{k-1}||\Delta y_{k-1}| +(c_\beta^2+a_\beta)|\Delta y_{k-1}|^2\\
    &\ge|\Delta\theta_{k-1}+c_\beta\Delta y_{k-1}|^2+a_\beta|\Delta y_{k-1}|^2
    = |(\Delta\theta_{k-1},\Delta y_{k-1})|_{P_\beta}^2
\end{align*}
}

As such, for
\begin{equation*}
    \kappa_{\gamma,\beta} =\min\left(\frac12, \lambda_{\min}\left(P_\beta^{-1/2}K_{\gamma,\beta}P_\beta^{-1/2}\right)\right), \qquad r_{\rm MS}(\gamma,\beta)=(1-\kappa_{\gamma,\beta})^{1/2},
\end{equation*}
we have
{
\[
    w_{k-1}^\top K_{\gamma,\beta}w_{k-1}
    \ge\kappa_{\gamma,\beta}w_{k-1}^\top P_\beta w_{k-1}
    \ge\kappa_{\gamma,\beta}|(\Delta\theta_{k-1},\Delta y_{k-1})|_{P_\beta}^2.
\]
}
Substitution into \eqref{eq:all-beta-decrement} gives
\begin{equation}
    \mathbb{E}\left[
        |(\Delta\theta_k,\Delta y_k)|_{P_\beta}^2
        \mid \mathcal{F}_{k-1}
    \right]
    \le
    (1-\kappa_{\gamma,\beta})
    |(\Delta\theta_{k-1},\Delta y_{k-1})|_{P_\beta}^2.
\end{equation}
For deterministic previous states, taking square roots gives the one-step bound~\eqref{eq:all-beta-gmc}. Conditioning and iteration give the forgetting bound.

The construction in Appendix~\ref{app:stationary-construction}, with $\dist=\dist_\beta$, $p=2$, and $r=r_{\rm MS}$, gives the unique invariant law and stationary causal solution.
\end{proof}

\subsection{Comparison of sufficient conditions}
\label{app:stability-comparision}
The second-moment calculation proves
\begin{equation}
\begin{split}
    &\|\dist_\beta(\Phi_{x_0}(z),\Phi_{x_0}(z'))\|_2 \le r_{\rm MS}(\gamma,\beta)\dist_\beta(z,z')
\end{split}
\label{eq:all-beta-gmc}
\end{equation}
The two conditions are sufficient for the same contraction property. The theorem covers every $p>1$, while the corollary gives a simple explicit bound when a second moment is available. For example, if $\mu=L_2=1$, $\beta=1/2$, and $\gamma=0.3$, the corollary applies since $\gamma_{\rm MS}=1/3$. Yet $\eta=0.6$ and $q_{\eta,2}=0.4$, so the second Perron inequality would require $0.18<0.12$. The bound~\eqref{eq:all-beta-step} is sufficient for nonlinear stochastic gradients. Appendix~\ref{app:sharp-gmc-proofs} gives the additive quadratic boundary.

\section{Proofs for averaging and the invariance principle}
\label{app:fclt-proofs}

\subsection{Proof of Proposition~\ref{prop:physical-dependence}}

\begin{proof}

Let $\theta_k^{(0)}=H(\mathcal X_k^{(0)})$, and let $y_k^{(0)}$ denote the corresponding momentum under the same innovation replacement. Define $Z_k^{(0)}=(\theta_k^{(0)},y_k^{(0)})$. Write $\dist$ for the norm distance in Theorem~\ref{thm:perron-gmc}, Corollary~\ref{thm:all-beta-gmc}, or Proposition~\ref{thm:curvature-rate}, and choose $c_\theta$ such that $|\theta-\theta'|\le c_\theta \dist(z,z')$. The two processes use the same observations from time $1$ onward, so the coordinate bound and GMC give
{
\begin{align*}
    \delta_{k,p}
    &=\|\theta_k^\circ-\theta_k^{(0)}\|_p\\
    &\le c_\theta\|\dist(Z_k^\circ,Z_k^{(0)})\|_p\\
    &\le c_\theta r^k\|\dist(Z_0^\circ,Z_0^{(0)})\|_p=:C_pr^k.
\end{align*}
}
Since $Z_0^{(0)}$ and $Z_0^\circ$ have the same stationary law, the triangle inequality and the stationary moment bound yield
{
\[
    C_p\le2c_\theta\|\dist(Z_0^\circ,z^\star)\|_p\le\frac{2c_\theta D_p}{1-r}<\infty,
\]
}
where $z^\star=(\theta^\star,0)$ and $D_p=\|\dist(\Phi_{x_0}(z^\star),z^\star)\|_p$; see Appendix~\ref{app:stationary-construction} or~\eqref{eq:curvature-displacement}.

Summing the geometric tail gives $\Theta_{m,p}=O(r^m)$.
\end{proof}

\subsection{Proof of Theorem~\ref{thm:averaged-sgdm}}

\begin{proof}
Write $U_k=\theta_k^\circ-\theta_\infty^\star(\gamma,\beta)$.
By Proposition~\ref{prop:physical-dependence}, the functional dependence coefficients are summable:
\eqs*{
    \sum_{j=0}^\infty \delta_{j, p} \le \frac{C_p}{1 - r}.
}
Recall $\mathcal{P}_jX = \mathbb{E}[X\mid\mathcal F_j^\circ]-\mathbb{E}[X\mid\mathcal F_{j-1}^\circ]$.
Then $\mathcal P_0 U_k = \E [U_k - U_k^{(0)} \mid \mathcal F_0^\circ]$ satisfies
\eqs*{
    \sum_{k = 0}^\infty \|\mathcal P_0 U_k\|_p
    \le \sum_{k = 0}^\infty \|U_k - U_k^{(0)}\|_p
    = \sum_{k = 0}^\infty \delta_{k,p} < \infty.
}

Write $U_k=(U_{k,1},\ldots,U_{k,d})^\top$. Applying \citet[Theorem~1(iii)]{Wu2007Strong} to $U_{k,j}$ for $j=1,\ldots,d$ gives
\begin{align*}
    \left\|\max_{m\le n}\left|\sum_{k=1}^mU_k\right|\right\|_p
    &\le\sum_{j=1}^d\left\|\max_{m\le n}\left|\sum_{k=1}^mU_{k,j}\right|\right\|_p
    \lesssim n^{1/(p\wedge2)}.
\end{align*}

The bound $\sum_{j\ge m}\|\mathcal P_0U_j\|_p\le\Theta_{m,p}=O(r^m)$ satisfies the conditions of \citet[Corollary~2(iii)]{Wu2007Strong}, so the Marcinkiewicz--Zygmund strong law gives, for $1<p<2$,
\eqs*{
    \sum_{k=1}^nU_k=o_{\rm a.s.}(n^{1/p}).
}
For $p>2$, the same decay verifies condition (15) of \citet[Theorem~2(i)]{Wu2007Strong}. For $p=2$, his Theorem~2(ii) uses condition (9). These results give the law-of-the-iterated-logarithm upper bound:
\eqs*{
    \sum_{k=1}^nU_k=O_{\rm a.s.}(\sqrt{n\log\log n})
}
Division by $n$ proves the assertions for the stationary process.

Lastly, to consider $\bar\theta_n$, write
\[
    \bar\theta_n-\theta_\infty^\star=\frac1n\sum_{k=1}^n(\theta_k^\circ-\theta_\infty^\star)+\frac1n\sum_{k=1}^n(\theta_k-\theta_k^\circ).
\]
The first term has the rates proved above. For the second term, the coordinate bound $|\theta-\theta'|\le c_\theta \dist(z,z')$ and GMC give
{
\[
    \sum_{k\ge1}\|\theta_k-\theta_k^\circ\|_p
    \le \frac{c_\theta r}{1-r}\|\dist(Z_0,Z_0^\circ)\|_p<\infty.
\]
}
Since $p>1$, the same series is summable in $L^1$, and Tonelli's theorem gives $\sum_{k\ge1}|\theta_k-\theta_k^\circ|<\infty$ almost surely.
Thus the second term is $O_{L^p}(n^{-1})$ and $O_{\rm a.s.}(n^{-1})$, smaller than all the stated rates. The same conclusions therefore hold for $\bar\theta_n$ from every deterministic initial state.

\end{proof}

\subsection{Proof of Theorem~\ref{thm:fclt}}

\begin{proof}

Let $U_k=\theta_k^\circ-\theta_\infty^\star(\gamma,\beta)$. We first write the partial sum as a martingale sum plus an error. Then we use the vector martingale FCLT. By Proposition~\ref{prop:physical-dependence} and the projection bound,
\[
    \|\mathbb E(U_j\mid\mathcal F_0^\circ)\|_2\le\sum_{\ell\ge j}\|\mathcal P_0U_\ell\|_2\le\sum_{\ell\ge j}\delta_{\ell,2}=O(r^j).
\]
Thus $H_k:=\sum_{j\ge1}\mathbb E(U_{k+j}\mid\mathcal F_k^\circ)$ converges in $L^2$. Define $D_k:=U_k+H_k-H_{k-1}$. Then $D_k\in L^2$ since $U_k,H_k,H_{k-1}\in L^2$.
The sequence $(D_k)$ is stationary and ergodic because it is a function of the iid innovations.
Since $\mathbb E(H_k\mid\mathcal F_{k-1}^\circ)=H_{k-1}-\mathbb E(U_k\mid\mathcal F_{k-1}^\circ)$,
we have
\[
    \mathbb E(D_k\mid\mathcal F_{k-1}^\circ)=0,\qquad
    \sum_{k=1}^mU_k=\sum_{k=1}^mD_k+H_0-H_m.
\]

We next check the conditions of the martingale FCLT. For every $\epsilon>0$, ergodicity and $D_0\in L^2$ give
{
\begin{align*}
    \frac1n\sum_{k=1}^{\lfloor ns\rfloor}D_kD_k^\top&\to s\mathbb E(D_0D_0^\top)\quad\text{almost surely},\qquad 0\le s\le1,\\
    \frac1n\mathbb E\max_{1\le k\le n}|D_k|^2&\le\epsilon^2+\mathbb E\big[|D_0|^2\mathbf1_{\{|D_0|>\epsilon\sqrt n\}}\big]\to\epsilon^2.
\end{align*}
}
Let $\epsilon\downarrow0$ in the second bound. Then the expected maximum jump of $n^{-1/2}\sum_{k\le\lfloor ns\rfloor}D_k$ goes to zero. These two bounds verify the conditions of the vector martingale FCLT \citep[Theorem~2.1(i)]{Whitt2007MartingaleFCLT}; see also \citet{Hall1980Martingale}.
Hence,
{
\[
    \left(\frac1{\sqrt n}\sum_{k=1}^{\lfloor ns\rfloor}D_k:0\le s\le1\right)
    \Rightarrow\left(\bigl(\mathbb E(D_0D_0^\top)\bigr)^{1/2}B(s):0\le s\le1\right)
    \quad\text{in }D([0,1],\mathbb R^d).
\]
}

For the partial sums of $U_k$ and $D_k$ to have the same process limit, we need $\max_{0\le m\le n}|H_0-H_m|/\sqrt n\to_{\mathbb P}0$. The triangle inequality gives
\[
    \max_{0\le m\le n}|H_0-H_m|\le |H_0|+\max_{0\le m\le n}|H_m|\le2\max_{0\le m\le n}|H_m|.
\]
The last maximum includes $m=0$. By stationarity and the union bound, for every $\epsilon>0$,
\begin{align*}
    \mathbb P\!\left(\frac{\max_{0\le m\le n}|H_0-H_m|}{\sqrt n}>\epsilon\right)
    &\le(n+1)\mathbb P\!\left(|H_0|>\frac{\epsilon\sqrt n}{2}\right)\\
    &\le\frac{4(n+1)}{\epsilon^2n}\mathbb E\big[|H_0|^2\mathbf1_{\{|H_0|>\epsilon\sqrt n/2\}}\big]\to0.
\end{align*}
The last limit follows from $H_0\in L^2$.

We now show that $\mathbb E(D_0D_0^\top)=\Sigma_{\gamma,\beta}^{\mathrm{av}}$. Since projections at different times are orthogonal, for $h\ge0$,
\[
    |\operatorname{Cov}(U_0,U_h)|\le\sum_{j\ge0}\|\mathcal P_0U_j\|_2\|\mathcal P_0U_{j+h}\|_2=O(r^h).
\]
Thus the covariance series is absolutely summable. The error satisfies $\|H_0-H_n\|_2\le2\|H_0\|_2$. Using the partial-sum formula above and the orthogonality of the martingale differences, we get
{
\[
    \mathbb E(D_0D_0^\top)=\lim_{n\to\infty}\frac1n\mathbb E\left[\left(\sum_{k=1}^nU_k\right)\left(\sum_{k=1}^nU_k\right)^\top\right]=\sum_{h\in\mathbb Z}\operatorname{Cov}(U_0,U_h)=\Sigma_{\gamma,\beta}^{\mathrm{av}}.
\]
}
By the martingale FCLT, the error bound, and Slutsky's theorem,

\eqs*{
    \left( \frac{1}{\sqrt n} \sum_{k=1}^{\lfloor ns\rfloor}U_k : 0 \le s\le1 \right)
    \Rightarrow \left( \bigl(\Sigma_{\gamma,\beta}^{\mathrm{av}}\bigr)^{1/2}B(s) :0\le s\le1 \right)
}
in $D([0,1],\mathbb R^d)$.

Lastly, let $\theta_k$ start from a deterministic state. Use the same future observations for $\theta_k$ and $\theta_k^\circ$. For $p=2$, the triangle inequality and GMC from Theorem~\ref{thm:perron-gmc}, Corollary~\ref{thm:all-beta-gmc}, or Proposition~\ref{thm:curvature-rate} give
{
\eqs*{
    \left\|\max_{1\le \ell\le n} \left|\sum_{k=1}^\ell(\theta_k-\theta_k^\circ)\right|\right\|_2
    &\le \sum_{k=1}^n\|\theta_k-\theta_k^\circ\|_2 \\
    &\le C\sum_{k=1}^n \|\dist(Z_k, Z_k^\circ) \|_2 \\
    &\le \frac{Cr}{1-r}\|\dist(Z_0,Z_0^\circ)\|_2
}
}
The second inequality uses $|\theta_k-\theta_k^{\circ}|\le C\dist(Z_k,Z_k^{\circ})$. Dividing by $\sqrt n$ gives
{
\eqs*{
    \left\|\max _{0 \leq s \leq 1}\left|\frac{1}{\sqrt{n}} \sum_{k=1}^{\lfloor n s\rfloor}\left(\theta_k-\theta_k^{\circ}\right)\right|\right\|_2 \lesssim \frac{1}{\sqrt{n}} \rightarrow 0,
}
}
Thus the two partial sum processes differ by $o_{\mathbb P}(1)$ in $D([0,1],\mathbb R^d)$.

Applying Slutsky gives the result:
\eqs*{
    \left(\frac{1}{\sqrt{n}} \sum_{k=1}^{\lfloor n s\rfloor}\left(\theta_k-\theta_{\infty}^{\star}\right): 0 \leq s \leq 1\right)
    \Rightarrow
    \left(\left(\Sigma_{\gamma, \beta}^{\mathrm{av}}\right)^{1 / 2} B(s): 0 \leq s \leq 1\right).
}
\end{proof}

\subsection{Proof of Corollary~\ref{coro:average-clt}}

\begin{proof}
From Theorem~\ref{thm:fclt}, \eqref{eq:invariance_principle}, consider marginal distribution at $s=1$.
Then by continuous mapping theorem,
\eqs*{
    \sqrt n (\bar \theta_n - \theta_\infty^\star)
    =  \frac{1}{\sqrt{n}} \sum_{k=1}^n\left(\theta_k-\theta_{\infty}^{\star}\right) \Rightarrow \N\left(0, \Sigma_{\gamma, \beta}^{\mathrm{av}} \right).
}

\end{proof}

\section{Proof of the last-iterate limit}
\label{app:last-iterate-proofs}

\subsection{Localization and linearization}

Write $e_k=\theta_k-\theta^\star$, $\eta=\gamma/(1-\beta)$, and $q_\eta=1-\mu\eta$. The following lemma collects only the localization and impulse-response estimates used below.

\begin{lemma}[Auxiliary bounds]
\label{lem:last-iterate-auxiliary}
Suppose Assumptions~\ref{ass:global} and~\ref{ass:local-hessian} hold with $p=2$, $\beta\in[0,1)$ is fixed, $\theta_0$ is deterministic, and $y_0=0$. For all sufficiently small $\gamma$, there is $C_\beta<\infty$, independent of $(k,\gamma)$, such that
\begin{equation}
    \|e_k\|_2^2\le C_\beta\left(q_\eta^k\rho_2^2+\gamma\right), \qquad \rho_2=|\theta_0-\theta^\star|. \label{eq:last-iterate-localization}
\end{equation}
Let
\[
    B_\gamma=
    \begin{pmatrix}
      I_d-\gamma A&\beta I_d\\
      -\gamma A&\beta I_d
    \end{pmatrix},
    \qquad
    C_\gamma=
    \begin{pmatrix}
      -\gamma I_d\\
      -\gamma I_d
    \end{pmatrix},
    \qquad
    H_{j,\gamma}=(I_d,0)B_\gamma^jC_\gamma.
\]
Then
\begin{equation}
    |H_{j,\gamma}|\le C_\beta\gamma q_\eta^{j+1}, \qquad
    \sum_{j\ge0}|H_{j,\gamma}|\le C_\beta, \qquad
    \sum_{j\ge0}|H_{j,\gamma}|^2\le C_\beta\gamma. \label{eq:last-iterate-impulse-bounds}
\end{equation}
For every fixed $s>0$,
\begin{equation}
    \gamma^{-1}H_{\lfloor s/\gamma\rfloor,\gamma}
    \to-\frac{1}{1-\beta}e^{-As/(1-\beta)}. \label{eq:last-iterate-impulse-limit}
\end{equation}
\end{lemma}

\begin{proof}
Put $u_k=\theta_k+\beta y_k/(1-\beta)$. The SGDM recursion gives
\[
    u_k=u_{k-1}-\eta g(\theta_{k-1},x_k).
\]
Use a constant weight on the momentum coordinate. Put $b=\beta/(1-\beta)$, $\delta=1-\beta$, and
\[
    V_k=|u_k-\theta^\star|^2+|y_k|^2,\qquad
    D=1+\delta^2,\qquad K=(b+\beta\delta)L_2+\mu b.
\]
Write $e=e_{k-1}$ and $y=y_{k-1}$. Since $u_{k-1}-\theta^\star=e+by$, the conditional expansion and Assumption~\ref{ass:global} give
\begin{align*}
    \mathbb E[V_k\mid\mathcal F_{k-1}]-V_{k-1}
    &=-2\eta\langle e,G(\theta_{k-1})\rangle
      -2\eta(b+\beta\delta)\langle y,G(\theta_{k-1})\rangle
      -(1-\beta^2)|y|^2\\
    &\quad +D\eta^2\mathbb E[|g(\theta_{k-1},x_k)|^2\mid\mathcal F_{k-1}]\\
    &\le-2\mu\eta|e|^2+2\eta(b+\beta\delta)L_2|e||y|
      -(1-\beta^2)|y|^2+2D\eta^2(L_2^2|e|^2+M_2^2).
\end{align*}
Adding $\mu\eta V_{k-1}$ and using
$2\eta K|e||y|\le(\mu\eta/2)|e|^2+(2\eta K^2/\mu)|y|^2$ yields
\begin{align*}
    \mathbb E[V_k\mid\mathcal F_{k-1}]-(1-\mu\eta)V_{k-1}
    &\le\left(-\frac{\mu\eta}{2}+2DL_2^2\eta^2\right)|e|^2\\
    &\quad+\left[-(1-\beta^2)+\eta\left(\mu(b^2+1)+\frac{2K^2}{\mu}\right)\right]|y|^2
      +2DM_2^2\eta^2.
\end{align*}
The coefficients of $|e|^2$ and $|y|^2$ are nonpositive whenever
\[
    0<\eta\le\min\left(
       \frac1{2\mu},\frac{\mu}{4DL_2^2},
       \frac{1-\beta^2}{\mu(b^2+1)+2K^2/\mu}
    \right).
\]
Thus $\mathbb E V_k\le q_\eta\mathbb E V_{k-1}+2DM_2^2\eta^2$. Since $V_0=\rho_2^2$ and
$|e_k|^2\le2\max(1,b^2)V_k$, iteration proves~\eqref{eq:last-iterate-localization}. All constants depend only on fixed parameters, and no additional moment condition is used.

Diagonalize $A=Q\diag(\lambda_i)Q^\top$. For $\lambda\in[\mu,L_2]$, set $a_\lambda=1+\beta-\gamma\lambda$ and let
\[
    r_s=\frac{a_\lambda+\sqrt{a_\lambda^2-4\beta}}2,
    \qquad r_f=\frac{a_\lambda-\sqrt{a_\lambda^2-4\beta}}2.
\]
If $\gamma L_2\le(1-\beta)^2/[4(1+\beta)]$, the roots are nonnegative and
$r_s-r_f\ge(1-\beta)/2$. Put $t=\eta\lambda$. At $x=1-t$ the characteristic polynomial satisfies
\[
    x^2-a_\lambda x+\beta=\beta t^2\ge0,
    \qquad x-\frac{a_\lambda}{2}
      =\frac{1-\beta-(1+\beta)t}{2}>0.
\]
Consequently $r_s\le1-\eta\lambda\le q_\eta$. The exact scalar impulse response is
\[
    h_{j,\gamma,\lambda}
    =-\gamma\frac{r_s^{j+1}-r_f^{j+1}}{r_s-r_f},
    \qquad
    |h_{j,\gamma,\lambda}|
    \le\frac{2\gamma}{1-\beta}q_\eta^{j+1}.
\]
Summing the bound and its square proves~\eqref{eq:last-iterate-impulse-bounds}. Finally,
$r_s=1-\gamma\lambda/(1-\beta)+O_\beta(\gamma^2)$, $r_f\to\beta$, and $r_s-r_f\to1-\beta$. At $j=\lfloor s/\gamma\rfloor$ with fixed $s>0$, the slow-root power tends to $e^{-\lambda s/(1-\beta)}$ and the fast-root power vanishes. The scalar formula and the fixed diagonalization prove~\eqref{eq:last-iterate-impulse-limit}.
\end{proof}

Let
\[
    \varepsilon_k(\theta)=g(\theta,x_k)-G(\theta),
    \qquad
    \varepsilon_k^\star=\varepsilon_k(\theta^\star)=g(\theta^\star,x_k).
\]
The local Taylor bound in Assumption~\ref{ass:local-hessian} and the global Lipschitz bound on $G$ give
\[
    |G(\theta)-A(\theta-\theta^\star)|\le L_G|\theta-\theta^\star|^2.
\]
Define
\[
    R_k=g(\theta_{k-1},x_k)-Ae_{k-1}-\varepsilon_k^\star,
    \qquad
    e_n^{\rm lin}=\sum_{k=1}^nH_{n-k,\gamma}\varepsilon_k^\star.
\]
Iteration of the recursion gives
{
\begin{equation}
    e_n=(I_d,0)B_\gamma^n(e_0,0)^\top+e_n^{\rm lin}+\sum_{k=1}^nH_{n-k,\gamma}R_k. \label{eq:last-iterate-decomposition}
\end{equation}
}

\begin{lemma}[Linearization error]
\label{lem:linear-approximation}
Under the conditions of Theorem~\ref{thm:last-iterate-clt}, if $n=n(\gamma)$ satisfies \eqref{eq:last-iterate-burnin}, then
\[
    \frac{\|e_n-e_n^{\rm lin}\|_1}{\sqrt\gamma}\to0.
\]
\end{lemma}

\begin{proof}
Split $R_k=R_k^{(1)}+R_k^{(2)}$, where
\[
    R_k^{(1)}=G(\theta_{k-1})-Ae_{k-1},
    \qquad
    R_k^{(2)}=\varepsilon_k(\theta_{k-1})-\varepsilon_k(\theta^\star).
\]
The homogeneous term in \eqref{eq:last-iterate-decomposition} is bounded by $C_\beta q_\eta^n\rho_2$. By \eqref{eq:last-iterate-localization} and \eqref{eq:last-iterate-impulse-bounds},
\[
    \left\|\sum_{k=1}^nH_{n-k,\gamma}R_k^{(1)}\right\|_1
    \le C_\beta\left(n\gamma q_\eta^n\rho_2^2+\gamma\right).
\]
The sequence $(R_k^{(2)},\mathcal F_k)$ is a martingale difference, and Assumption~\ref{ass:global} gives
\[
    \mathbb{E}\left(|R_k^{(2)}|^2\mid\mathcal F_{k-1}\right)\le L_2^2|e_{k-1}|^2.
\]
Martingale orthogonality, \eqref{eq:last-iterate-localization}, and \eqref{eq:last-iterate-impulse-bounds} yield
\[
    \left\|\sum_{k=1}^nH_{n-k,\gamma}R_k^{(2)}\right\|_2
    \le C_\beta\left(\sqrt\gamma q_\eta^{n/2}\rho_2+\gamma\right).
\]
Consequently,
\[
    \frac{\|e_n-e_n^{\rm lin}\|_1}{\sqrt\gamma}
    \le C_\beta\left(
      \frac{q_\eta^n}{\sqrt\gamma}\rho_2
      +n\sqrt\gamma q_\eta^n\rho_2^2
      +q_\eta^{n/2}\rho_2
      +\sqrt\gamma
    \right).
\]
The first term is bounded by $(n\gamma)^{-1}n\sqrt\gamma q_\eta^{n-1}\rho_2$, up to a fixed constant. Hence all four terms converge to zero under \eqref{eq:last-iterate-burnin}.
\end{proof}

\subsection{Proof of the last-iterate CLT}

\begin{proof}[Proof of Theorem~\ref{thm:last-iterate-clt}]
Let $S=\mathbb{E}[\varepsilon_0^\star(\varepsilon_0^\star)^\top]$. From \eqref{eq:last-iterate-impulse-limit}, the exponential envelope in \eqref{eq:last-iterate-impulse-bounds}, and $n\gamma\to\infty$, dominated convergence for the Riemann sums gives
{
\begin{align*}
    \frac1\gamma\Cov(e_n^{\rm lin})
    &=\frac1\gamma\sum_{j=0}^{n-1}H_{j,\gamma}S H_{j,\gamma}^\top
    \to
      \frac{1}{(1-\beta)^2}\int_0^\infty
      e^{-As/(1-\beta)}Se^{-As/(1-\beta)}\,ds
      =\frac{\Gamma}{1-\beta}.
\end{align*}
}

For a fixed nonzero $v\in\R^d$, consider the independent centered triangular array
{
\[
    X_{n,k}=\frac1{\sqrt\gamma}v^\top H_{n-k,\gamma}\varepsilon_k^\star,
    \qquad 1\le k\le n.
\]
}
The covariance convergence above gives
{
\[
    \sum_{k=1}^n\mathbb{E}[X_{n,k}^2]\to\frac{v^\top\Gamma v}{1-\beta}.
\]
}
Moreover, \eqref{eq:last-iterate-impulse-bounds} implies
\[
    \max_{0\le j<n}\frac{|H_{j,\gamma}|}{\sqrt\gamma}\le C_\beta\sqrt\gamma,
    \qquad
    \frac1\gamma\sum_{j=0}^{n-1}|H_{j,\gamma}|^2\le C_\beta.
\]
Thus, for every $\delta>0$, uniform integrability of $|\varepsilon_0^\star|^2$ gives
\eqs*{
    \sum_{k=1}^n\mathbb{E}\left[X_{n,k}^2\ind{|X_{n,k}|>\delta}\right]\le\frac{|v|^2}{\gamma}\sum_{j=0}^{n-1}|H_{j,\gamma}|^2\,
    \mathbb{E}\left[|\varepsilon_0^\star|^2\ind{|\varepsilon_0^\star|>\delta/(C_\beta|v|\sqrt\gamma)}\right]
    \to0.
}
The Lindeberg--Feller theorem and the Cram\'er--Wold device give
\[
    \frac{e_n^{\rm lin}}{\sqrt\gamma}\Rightarrow N\!\left(0,\frac{\Gamma}{1-\beta}\right).
\]
Lemma~\ref{lem:linear-approximation} and Slutsky's theorem now prove \eqref{eq:last-iterate-clt}.
\end{proof}

\section{Additive quadratic benchmarks}
\label{app:sharp-gmc-proofs}

\subsection{Classical additive quadratic benchmark}
\label{app:additive-benchmark}
The GMC results in Section~\ref{sec:gmc} give sufficient conditions for nonlinear stochastic gradients. For an additive-noise quadratic model, synchronous coupling cancels the innovations exactly. We can then compute the GMC rate and determine how the rate depends on the learning rate and momentum.

The spectral calculation is the classical heavy-ball calculation \citep{polyak1964some,goujaud2023nonacceleration}. Here it identifies the exponential rate at which the stochastic recursion forgets its initialization. This identification uses a fixed Hessian and additive noise. Random curvature does not cancel under the same coupling.

Let $A$ be symmetric with $\mu I_d\preceq A\preceq L I_d$, and suppose
{
\begin{equation*}
    g(\theta,x_k)=A(\theta-\theta^\star)+\varepsilon_k,
\end{equation*}
}
where $(\varepsilon_k)$ is i.i.d. with mean zero. Writing $e_k=\theta_k-\theta^\star$ and eliminating $y_k$ from \eqref{eq:sgdm} gives
{
\begin{equation*}
    e_k=\left((1+\beta)I_d-\gamma A\right)e_{k-1}-\beta e_{k-2}-\gamma\varepsilon_k.
\end{equation*}
}
For two synchronously coupled trajectories, let $\Delta e_k=e_k-e_k'$. In an eigendirection of $A$ with eigenvalue $\lambda$, their lag difference follows the matrix
\begin{equation}
    T_\lambda(\gamma,\beta)=\begin{pmatrix}1+\beta-\gamma\lambda&-\beta\\1&0\end{pmatrix}. \label{eq:quadratic-block}
\end{equation}
Its characteristic polynomial is $z^2-a_\lambda z+\beta$, where $a_\lambda=1+\beta-\gamma\lambda$. Thus its spectral radius is $\psi_\beta(a_\lambda)$, with
\begin{equation}
    \psi_\beta(a)=\begin{cases}\sqrt\beta, & |a|\le2\sqrt\beta,\\[2mm]\dfrac{|a|+\sqrt{a^2-4\beta}}{2}, & |a|>2\sqrt\beta.\end{cases} \label{eq:psi-beta}
\end{equation}
Define the worst-case asymptotic contraction factor by $\rho_{\rm quad}(\gamma,\beta):=\max_{\lambda\in[\mu,L]}\rho(T_\lambda(\gamma,\beta))$.

\begin{theorem}[Exact all-$\beta$ quadratic GMC rate]
\label{thm:exact-quadratic-rate}
For every $\beta\in[0,1)$ and $\gamma>0$,
\begin{equation}
    \rho_{\rm quad}(\gamma,\beta)=\max_{\lambda\in\{\mu,L\}}\psi_\beta(1+\beta-\gamma\lambda). \label{eq:exact-quadratic-rate}
\end{equation}
Moreover, $\rho_{\rm quad} (\gamma, \beta) < 1$ if and only if
\begin{equation}
    0<\gamma<\frac{2(1+\beta)}{L}. \label{eq:exact-quadratic-stability}
\end{equation}
Under this condition, for every $r\in(\rho_{\rm quad}(\gamma,\beta),1)$,
there exists a metric $\dist_r$ such that, for all deterministic $z,z'\in\R^{2d}$,
\eqs{
    \|\dist_r(\Phi_{x_k}(z),\Phi_{x_k}(z'))\|_p \le r \dist_r(z,z').
}
\end{theorem}

The metric $\dist_r$ is adapted to the fixed matrix $A$. Exactness refers to the worst asymptotic spectral rate over the quadratic class. The proof in Appendix~\ref{app:exact-quadratic-proof} gives the matching lower bound for contraction in an equivalent norm. At a repeated root the geometric decay term may have a polynomial prefactor, which is why the one-step bound allows any $r>\rho_{\rm quad}$.

\begin{remark}
For $\beta>0$, the SGDM stability interval in~\eqref{eq:exact-quadratic-stability} extends beyond the SGD interval $0<\gamma<2/L$.
This compares methods at the same unnormalized step $\gamma$: SGD requires $\gamma<2/L$. Figure~\ref{fig:numerical-report}(c) instead compares SGDM and SGD at their respective optimal tunings for quadratic forgetting.
\end{remark}

\subsection{Optimization over the learning rate and momentum}

We now optimize the quadratic GMC rate over $(\gamma,\beta)$. Let $\kappa=L/\mu$. Optimally tuned gradient descent has contraction rate
\begin{equation}
    \begin{aligned}
    \rho_{\rm SGD}^\star
    &:=\min_{\gamma>0}\max_{\lambda\in[\mu,L]}|1-\gamma\lambda|
    =\frac{L-\mu}{L+\mu}.
    \end{aligned}
    \label{eq:optimalGD}
\end{equation}
With momentum, we additionally optimize over $\beta$:
\[
    \rho_{\rm SGDM}^\star:=\min_{0\le\beta<1}\min_{\gamma>0}\rho_{\rm quad}(\gamma,\beta).
\]
The following tuning agrees with \citet[Corollary 2.3]{goujaud2023nonacceleration}.

\begin{corollary}[Optimal rate at a fixed momentum]
\label{cor:fixed-beta-optimum}
For each fixed $\beta\in[0,1)$, a learning rate that minimizes $\rho_{\rm quad}(\gamma,\beta)$ is $\gamma_{\rm bal}=2(1+\beta)/(L+\mu)$. Write $\tau_\beta=(1+\beta)\rho_{\rm SGD}^\star$ and define
\[
    \beta_\star
    :=\left(\frac{\sqrt\kappa-1}{\sqrt\kappa+1}\right)^2.
\]
The corresponding contraction rate is
\begin{equation}
    \rho_\star(\beta)=\begin{cases}
        \dfrac{\tau_\beta+\sqrt{\tau_\beta^2-4\beta}}{2},
            &0\le\beta<\beta_\star,\\[3mm]
        \sqrt\beta, &\beta_\star\le\beta<1.
    \end{cases}
    \label{eq:fixed-beta-rate}
\end{equation}
Additionally optimizing over $\beta$, the joint optimum is
\begin{equation}
    \gamma_\star=\frac{4}{(\sqrt L+\sqrt\mu)^2}, \qquad
    \eta_\star=\frac1{\sqrt{\mu L}},
    \label{eq:joint-optimum}
\end{equation}
and
\begin{equation}
    \rho_{\rm SGDM}^\star=\rho_\star(\beta_\star)=\frac{\sqrt\kappa-1}{\sqrt\kappa+1}=1-\Theta(\kappa^{-1/2}).
    \label{eq:accelerated-rate}
\end{equation}
\end{corollary}

From \eqref{eq:optimalGD}, optimally tuned gradient descent has rate $\rho_{\rm SGD}^\star=1-\Theta(\kappa^{-1})$ with $\gamma_{\rm GD}^\star=2/(L+\mu)$. Thus \eqref{eq:accelerated-rate} implies that the initialization forgetting scale improves from order $\kappa$ to order $\sqrt\kappa$. More precisely, $\rho_\star(\beta)<\rho_{\rm SGD}^\star$ for $0<\beta<(\rho_{\rm SGD}^\star)^2$, equality holds at $\beta=0$ and $\beta=(\rho_{\rm SGD}^\star)^2$, and momentum is slower for $\beta>(\rho_{\rm SGD}^\star)^2$.

For any $\beta<1$ fixed independently of $\kappa$,
\[
    1-\rho_\star(\beta)=\frac{2(1+\beta)}{1-\beta}\kappa^{-1}+O_\beta(\kappa^{-2}).
\]
Thus a fixed positive momentum improves the leading constant, while the order remains $\kappa^{-1}$. The optimal $\kappa^{-1/2}$ order uses condition-number-dependent momentum with $1-\beta_\star=\Theta(\kappa^{-1/2})$.

\subsection{Proof of Theorem~\ref{thm:exact-quadratic-rate}}
\label{app:exact-quadratic-proof}

\begin{proof}
Under synchronous coupling, the additive noise cancels, so the lag difference $\xi_k=(\Delta e_k,\Delta e_{k-1})$ satisfies
\[
    \xi_k=T_A\xi_{k-1},\qquad T_A=\begin{pmatrix}(1+\beta)I_d-\gamma A&-\beta I_d\\I_d&0\end{pmatrix}.
\]
Diagonalizing the symmetric matrix $A$ and reordering the lag coordinates makes $T_A$ orthogonally similar to the direct sum of the blocks $T_\lambda$, $\lambda\in\sigma(A)$. Consequently,
\[
    \rho(T_A)=\max_{\lambda\in\sigma(A)}\rho(T_\lambda)\le\rho_{\rm quad}(\gamma,\beta).
\]
From \eqref{eq:psi-beta}, the spectral radius of $T_\lambda$ is $\psi_\beta(a_\lambda)$.  The function $\psi_\beta(a)$ is nondecreasing in $|a|$, while $|1+\beta-\gamma\lambda|$ is convex in $\lambda$.  Its maximum on $[\mu,L]$ is therefore attained at $\mu$ or $L$, which proves \eqref{eq:exact-quadratic-rate}.

We next characterize stability.
The Jury conditions for the characteristic equation $z^2-a_\lambda z+\beta$ to have both roots in the open unit disk are
\eqs*{
      1-\beta>0, \qquad \gamma\lambda>0, \qquad 2(1+\beta)-\gamma\lambda>0,
}
which gives the range of $\gamma$ in \eqref{eq:exact-quadratic-stability}.

For GMC, fix $\rho_{\rm quad}(\gamma,\beta)<r<1$ and define
\[
    |\xi|_r:=\sup_{j\ge0}r^{-j}|T_A^j\xi|.
\]
Since $\rho(T_A)<r$, the spectral-radius formula ensures $|\xi|\le|\xi|_r\le C_r|\xi|$ for some $C_r<\infty$. Shifting the supremum gives $|T_A\xi|_r\le r|\xi|_r$. The invertible map $R(\theta,y)=(\theta,\theta-y)$ converts SGDM state differences to lag differences because $y_k=e_k-e_{k-1}$. Thus $\dist_r(z,z'):=|R(z-z')|_r$ is an equivalent state-space metric, and synchronous coupling yields
\[
    \|\dist_r(\Phi_{x_k}(z),\Phi_{x_k}(z'))\|_p=|T_A R(z-z')|_r\le r \dist_r(z,z').
\]
The equality holds because the coupled difference is deterministic. This proves one-step $L^p$ GMC for every $\rho_{\rm quad}(\gamma,\beta)<r<1$.

Conversely, suppose a norm on the lag state satisfies $|T_A\xi|_*\le r|\xi|_*$ for every $\xi$. Iteration and equivalence of finite-dimensional norms give $|T_A^k|\le Cr^k$. The spectral-radius formula implies $\rho(T_A)\le r$. Combined with the construction above, this shows that the infimum of norm contraction factors for a fixed $A$ is $\rho(T_A)$. An endpoint maximizing~\eqref{eq:exact-quadratic-rate} can be an eigenvalue of an admissible $A$, so the supremum of these infima over the quadratic class is $\rho_{\rm quad}$. A repeated root may prevent attainment of the infimum. This lower bound applies to contraction in equivalent norms.

\end{proof}

\subsection{Proof of Corollary~\ref{cor:fixed-beta-optimum}}

\begin{proof}
We first minimize over $\gamma$ with $\beta$ fixed, and then minimize the resulting rate over $\beta$. Recall $\rho_{\rm SGD}^\star=(L-\mu)/(L+\mu)$ from \eqref{eq:optimalGD}, and set
\[
    a_\mu=1+\beta-\gamma\mu,\qquad a_L=1+\beta-\gamma L,\qquad M_\beta(\gamma)=\max(|a_\mu|,|a_L|).
\]
By \eqref{eq:exact-quadratic-rate}, $\rho_{\rm quad}(\gamma,\beta)=\psi_\beta(M_\beta(\gamma))$. Since $\psi_\beta$ is nondecreasing on $[0,\infty)$, it suffices to minimize $M_\beta$. The identity
\[
    (L-\mu)(1+\beta)=La_\mu-\mu a_L\le L|a_\mu|+\mu|a_L|\le(L+\mu)M_\beta(\gamma)
\]
gives the lower bound $M_\beta(\gamma)\ge(1+\beta)\rho_{\rm SGD}^\star$. Equality is attained by balancing the endpoint arguments:
\[
    a_\mu=-a_L\quad\Longleftrightarrow\quad\gamma=\gamma_{\rm bal}:=\frac{2(1+\beta)}{L+\mu}.
\]
At this learning rate, $a_\mu=(1+\beta)\rho_{\rm SGD}^\star$ and $a_L=-(1+\beta)\rho_{\rm SGD}^\star$. Thus, with $\tau_\beta=(1+\beta)\rho_{\rm SGD}^\star$,
\[
    \rho_\star(\beta):=\min_{\gamma>0}\rho_{\rm quad}(\gamma,\beta)=\psi_\beta(\tau_\beta).
\]
This proves the claimed choice of $\gamma_{\rm bal}$; uniqueness is not needed because $\psi_\beta$ has a constant branch.

We next determine which branch of \eqref{eq:psi-beta} applies. Suppose first that $L>\mu$, so $0<\rho_{\rm SGD}^\star<1$, and put $t=\sqrt\beta\in[0,1)$. The boundary condition is
\[
    \tau_\beta=2\sqrt\beta\quad\Longleftrightarrow\quad \rho_{\rm SGD}^\star t^2-2t+\rho_{\rm SGD}^\star=0.
\]
The two roots are $(1\pm\sqrt{1-(\rho_{\rm SGD}^\star)^2})/\rho_{\rm SGD}^\star$. Only the smaller root lies in $[0,1)$, and it equals
\[
    t_\star=\frac{1-\sqrt{1-(\rho_{\rm SGD}^\star)^2}}{\rho_{\rm SGD}^\star}=\frac{\sqrt L-\sqrt\mu}{\sqrt L+\sqrt\mu}=\frac{\sqrt\kappa-1}{\sqrt\kappa+1}.
\]
Hence, the polynomial $\rho_{\rm SGD}^\star t^2-2t+\rho_{\rm SGD}^\star$ is positive for $0\le t<t_\star$ and nonpositive for $t_\star\le t<1$.

Therefore, defining $\beta_\star=t_\star^2$ gives exactly \eqref{eq:fixed-beta-rate}:
\[
    \rho_\star(\beta)=\begin{cases}
        \dfrac{\tau_\beta+\sqrt{\tau_\beta^2-4\beta}}2,&0\le\beta<\beta_\star,\\[2mm]
        \sqrt\beta,&\beta_\star\le\beta<1.
    \end{cases}
\]

To optimize over $\beta$, consider the first branch and write $\rho=\rho_\star(\beta)$. It satisfies $\rho^2-(1+\beta)\rho_{\rm SGD}^\star\rho+\beta=0$. Differentiating with respect to $\beta$ yields
\[
    2\rho\rho'-\rho_{\rm SGD}^\star\rho-(1+\beta)\rho_{\rm SGD}^\star\rho'+1=0,\qquad
    \rho'=\frac{\rho_{\rm SGD}^\star\rho-1}{2\rho-\tau_\beta}.
\]
Here $2\rho-\tau_\beta=\sqrt{\tau_\beta^2-4\beta}>0$. Also, $\gamma_{\rm bal}<2(1+\beta)/L$, so \eqref{eq:exact-quadratic-stability} gives $\rho<1$. Since $\rho_{\rm SGD}^\star<1$, the numerator $\rho_{\rm SGD}^\star\rho-1$ is negative. Thus the first branch decreases with $\beta$, whereas the second branch $\sqrt\beta$ increases. The branches agree at $\beta_\star$, so the joint minimum occurs there.

It remains to evaluate the learning rates at this momentum. From the expression for $\beta_\star$,
\[
    1+\beta_\star=\frac{2(L+\mu)}{(\sqrt L+\sqrt\mu)^2},\qquad 1-\beta_\star=\frac{4\sqrt{\mu L}}{(\sqrt L+\sqrt\mu)^2}.
\]
Substituting into $\gamma_{\rm bal}$ and $\eta=\gamma/(1-\beta)$ proves the remaining identities in \eqref{eq:joint-optimum}:
\[
    \gamma_\star=\frac{2(1+\beta_\star)}{L+\mu}=\frac4{(\sqrt L+\sqrt\mu)^2},\qquad
    \eta_\star=\frac{\gamma_\star}{1-\beta_\star}=\frac1{\sqrt{\mu L}}.
\]
The minimum rate is
\[
    \rho_{\rm SGDM}^\star=\rho_\star(\beta_\star)=\sqrt{\beta_\star}=\frac{\sqrt\kappa-1}{\sqrt\kappa+1}=1-\frac2{\sqrt\kappa+1},
\]
which gives \eqref{eq:accelerated-rate}, including its $1-\Theta(\kappa^{-1/2})$ asymptotic form as $\kappa\to\infty$.

If $L=\mu$, then $\rho_{\rm SGD}^\star=0$ and $\tau_\beta=0$, so $\rho_\star(\beta)=\sqrt\beta$ for every $\beta\in[0,1)$. Its minimum is attained at $\beta_\star=0$, with $\gamma_\star=\eta_\star=1/\mu$ and $\rho_{\rm SGDM}^\star=0$, as asserted by the same formulas.
\end{proof}

\section{Proofs for stochastic curvature}
\label{app:rate-variance-finite-n}

This appendix gives the proofs for Section~\ref{sec:quadratic-acceleration} and a coordinate-sampled least-squares example. We use the commuting-curvature model~\eqref{eq:commuting-curvature-model}, with a fixed iid sampling distribution and only second moments. Curvature and additive noise need not be independent.

Write $e_{i,k}=(Q^\top(\theta_k-\theta^\star))_i$, $h_i=\lambda_i$, and use the variances $V_i$, $\sigma_i^2$, and $\Sigma_i^{\rm av}$ from Proposition~\ref{prop:curvature-variance}.

\subsection{Proof of Proposition~\ref{thm:curvature-rate}}
\label{app:curvature-rate}

\begin{proof}
We first identify the forgetting factor from the coupled second moments, then obtain GMC and stationarity. Lemma~\ref{lem:curvature-rate-criterion} below gives the scalar rate and stability criteria used in the remaining proofs.

\textit{1. Coupled second moments.}
Take two SGDM trajectories driven by the same innovations $(H_k,\varepsilon_k)$. Set $e_{i,-1}=e_{i,0}-(Q^\top y_0)_i$, and use the same convention for the primed trajectory. In coordinate $i$, eliminating the momentum from~\eqref{eq:sgdm} gives, for $k\ge1$,
\[
    e_{i,k}=(1+\beta-\gamma H_{i,k})e_{i,k-1}-\beta e_{i,k-2}-\gamma(Q^\top\varepsilon_k)_i.
\]
The primed trajectory satisfies the same recursion with the same additive noise. Write $\Delta e_{i,k}=e_{i,k}-e_{i,k}'$. Subtracting gives
\[
    \Delta e_{i,k}=(1+\beta-\gamma H_{i,k})\Delta e_{i,k-1}-\beta\Delta e_{i,k-2}.
\]
Thus the additive noise cancels, while the curvature remains random. In lag coordinates,
\[
    \begin{pmatrix}\Delta e_{i,k}\\\Delta e_{i,k-1}\end{pmatrix}
    =B_{i,k}\begin{pmatrix}\Delta e_{i,k-1}\\\Delta e_{i,k-2}\end{pmatrix},
    \qquad B_{i,k}=\begin{pmatrix}1+\beta-\gamma H_{i,k}&-\beta\\1&0\end{pmatrix}.
\]

Fix $i$ and write $h=\lambda_i$, $v=v_i$, $a=1+\beta-\gamma h$, and $w=\gamma^2v$. Let $\mathcal F_{k-1}$ contain the initial states and the innovations through time $k-1$. Independence gives
\[
    \mathbb E(1+\beta-\gamma H_{i,k})=a,\qquad
    \mathbb E((1+\beta-\gamma H_{i,k})^2)=a^2+w.
\]
Consequently,
\begin{align*}
    \mathbb E((\Delta e_{i,k})^2\mid\mathcal F_{k-1})
    &=\mathbb E((1+\beta-\gamma H_{i,k})^2)(\Delta e_{i,k-1})^2\\
    &\quad-2\beta\mathbb E(1+\beta-\gamma H_{i,k})
      \Delta e_{i,k-1}\Delta e_{i,k-2}+\beta^2(\Delta e_{i,k-2})^2\\
    &=(a^2+w)(\Delta e_{i,k-1})^2
      -2a\beta\Delta e_{i,k-1}\Delta e_{i,k-2}+\beta^2(\Delta e_{i,k-2})^2,\\
    \mathbb E(\Delta e_{i,k}\Delta e_{i,k-1}\mid\mathcal F_{k-1})
    &=a(\Delta e_{i,k-1})^2-\beta\Delta e_{i,k-1}\Delta e_{i,k-2}.
\end{align*}
Taking expectations and collecting the three moments gives
\begin{equation}
    \begin{pmatrix}\mathbb E((\Delta e_{i,k})^2)\\\mathbb E(\Delta e_{i,k}\Delta e_{i,k-1})\\\mathbb E((\Delta e_{i,k-1})^2)\end{pmatrix}
    =\underbrace{\begin{pmatrix}a^2+w&-2a\beta&\beta^2\\a&-\beta&0\\1&0&0\end{pmatrix}}_{\mathcal T(h,v)}
    \begin{pmatrix}\mathbb E((\Delta e_{i,k-1})^2)\\\mathbb E(\Delta e_{i,k-1}\Delta e_{i,k-2})\\\mathbb E((\Delta e_{i,k-2})^2)\end{pmatrix}.
    \label{eq:curvature-moment-operator}
\end{equation}
Iterating this recursion,
\[
    \mathbb E((\Delta e_{i,k})^2+(\Delta e_{i,k-1})^2)
    =(1,0,1)\mathcal T(h,v)^k
    \begin{pmatrix}\mathbb E((\Delta e_{i,0})^2)\\\mathbb E(\Delta e_{i,0}\Delta e_{i,-1})\\\mathbb E((\Delta e_{i,-1})^2)\end{pmatrix}.
\]
Thus powers of $\mathcal T(h,v)$ determine the coupled second moments. Write $R_i=\rho(\mathcal T(\lambda_i,v_i))$ and $R_{\max}=\max_iR_i$.

\textit{2. From second moments to the decay rate.}
Let $u=(u_1,u_2)^\top=(\Delta e_{i,0},\Delta e_{i,-1})^\top$ be the deterministic initial difference. By linearity, we may normalize $|u|=1$. We bound $\sup_{|u|=1}\mathbb E|B_{i,k}\cdots B_{i,1}u|^2$ above and below by multiples of $|\mathcal T(h,v)^k|$, with constants independent of $k$, to show that their exponential rates agree. First,
{
\[
    \mathbb E|B_{i,k}\cdots B_{i,1}u|^2
    =(1,0,1)\mathcal T(h,v)^k(u_1^2,u_1u_2,u_2^2)^\top
    \le\sqrt2|\mathcal T(h,v)^k|,
\]
}
because $|(u_1^2,u_1u_2,u_2^2)|\le u_1^2+u_2^2=1$.

For the reverse bound, the second-moment matrix of $B_{i,k}\cdots B_{i,1}u$ is positive semidefinite. For any positive semidefinite $2\times2$ matrix $S$,
\[
    S_{12}^2\le S_{11}S_{22},\qquad
    S_{11}^2+S_{12}^2+S_{22}^2\le(S_{11}+S_{22})^2.
\]
Applying this inequality to the moment recursion gives
\eqs*{
    |\mathcal T(h,v)^k(u_1^2,u_1u_2,u_2^2)^\top|
    &=\left(\sum_{j=1}^2(\mathbb E(B_{i,k}\cdots B_{i,1}u)_j^2)^2+(\mathbb E((B_{i,k}\cdots B_{i,1}u)_1(B_{i,k}\cdots B_{i,1}u)_2))^2\right)^{1/2}\\
    &\le\mathbb E(B_{i,k}\cdots B_{i,1}u)_1^2+\mathbb E(B_{i,k}\cdots B_{i,1}u)_2^2=\mathbb E|B_{i,k}\cdots B_{i,1}u|^2.
}
For $u=(1,0)^\top$ and $u=(0,1)^\top$, the initial moment vectors are $(1,0,0)^\top$ and $(0,0,1)^\top$. Thus the first and third columns satisfy
\[
    \max\left(
      |\mathcal T(h,v)^k(1,0,0)^\top|,
      |\mathcal T(h,v)^k(0,0,1)^\top|
    \right)
    \le\sup_{|u|=1}\mathbb E|B_{i,k}\cdots B_{i,1}u|^2.
\]
To control the second column, take $u=(1,\pm1)^\top/\sqrt2$. Their initial moment vectors satisfy
\[
    (1/2,1/2,1/2)^\top-(1/2,-1/2,1/2)^\top=(0,1,0)^\top.
\]
Hence
\begin{align*}
    |\mathcal T(h,v)^k(0,1,0)^\top|
    &=\left|\mathcal T(h,v)^k(1/2,1/2,1/2)^\top
      -\mathcal T(h,v)^k(1/2,-1/2,1/2)^\top\right|\\
    &\le2\sup_{|u|=1}\mathbb E|B_{i,k}\cdots B_{i,1}u|^2.
\end{align*}
The operator norm is at most the square root of the sum of the squared column norms. Hence
{
\begin{equation}
    \frac1{\sqrt2}\sup_{|u|=1}\mathbb E|B_{i,k}\cdots B_{i,1}u|^2
    \le|\mathcal T(h,v)^k|
    \le\sqrt6\sup_{|u|=1}\mathbb E|B_{i,k}\cdots B_{i,1}u|^2.
    \label{eq:curvature-moment-comparison}
\end{equation}
}
The spectral-radius formula gives
{
\[
    \lim_{k\to\infty}|\mathcal T(h,v)^k|^{1/k}=\rho(\mathcal T(h,v)).
\]
}
Taking $k$th roots in~\eqref{eq:curvature-moment-comparison} therefore gives
\[
    \lim_{k\to\infty}
    \left(\sup_{|u|=1}\mathbb E|B_{i,k}\cdots B_{i,1}u|^2\right)^{1/k}
    =\rho(\mathcal T(h,v))=R_i.
\]
Across coordinates,
\[
    \sup_{\sum_i|u_i|^2=1}\sum_i\mathbb E|B_{i,k}\cdots B_{i,1}u_i|^2
    =\max_i\sup_{|u|=1}\mathbb E|B_{i,k}\cdots B_{i,1}u|^2.
\]
Thus the largest asymptotic second-moment factor is $\max_iR_i=R_{\max}$, without requiring independence across coordinates.
The preceding limits imply that, for every $r>\sqrt{R_{\max}}$, there is a constant $C_r<\infty$ such that
\eqs*{
    \sup_{\sum_i|u_i|^2=1}\sum_i\mathbb E|B_{i,k}\cdots B_{i,1}u_i|^2\le C_r r^{2k},\qquad k\ge1.
}
Conversely, this bound cannot hold for $0<r<\sqrt{R_{\max}}$.

Since $y_k=\theta_k-\theta_{k-1}$, the lag state and $Z_k=(\theta_k,y_k)$ are related by an invertible linear transformation. Their norms differ by fixed factors, which do not affect the exponential rate. Taking square roots therefore identifies the forgetting factor defined in Section~\ref{sec:stochastic-curvature}:
\[
    q=\sqrt{R_{\max}}.
\]

\textit{3. GMC and stationarity.}
We now obtain the one-step $L^2$ contraction needed for Section~\ref{sec:averaging-fclt}. Fix $r\in(q,1)$, so $R_i<r^2$ for every coordinate. The rescaled lag difference satisfies
\eqs*{
    r^{-k}\begin{pmatrix}\Delta e_{i,k}\\\Delta e_{i,k-1}\end{pmatrix}
    =\frac{B_{i,k}}r\,r^{-(k-1)}\begin{pmatrix}\Delta e_{i,k-1}\\\Delta e_{i,k-2}\end{pmatrix}.
}
Replacing $B_{i,k}$ by $B_{i,k}/r$ divides every second-moment term in~\eqref{eq:curvature-moment-operator} by $r^2$. Thus the second-moment transition matrix and its spectral radius are
\eqs*{
    \frac{\mathcal T(\lambda_i,v_i)}{r^2},\qquad
    \rho\left(\frac{\mathcal T(\lambda_i,v_i)}{r^2}\right)=\frac{R_i}{r^2}<1.
}
Hence, $\mathbb E\left|r^{-k}B_{i,k}\cdots B_{i,1}u\right|^2\to 0$ for every $u\in\mathbb R^2$ implies mean-square stability of the rescaled recursion. The stochastic Lyapunov criterion \citep[Section~9.1.1]{boyd1994linear} therefore gives a positive definite matrix $P_i$ satisfying
\begin{equation}
    P_i\succ0,\qquad r^2P_i-\mathbb E(B_{i,1}^\top P_iB_{i,1})=I_2.
    \label{eq:curvature-adapted-metric}
\end{equation}
Thus, in the squared norm $u^\top P_i u$, the expected squared norm after one step is at most $r^2$ times its initial value. Since $y_k=\theta_k-\theta_{k-1}$, a state difference $(\theta-\theta',y-y')$ has previous parameter difference $\theta-\theta'-y+y'$. We measure these two parameter differences in each coordinate using $P_i$ and sum the resulting squared norms.
For $z=(\theta,y)$ and $z'=(\theta',y')$, define
{
\[
    \dist_r(z,z')^2
    =\sum_i
    \begin{pmatrix}(Q^\top(\theta-\theta'))_i\\(Q^\top(\theta-\theta'-y+y'))_i\end{pmatrix}^{\!\top}
    P_i
    \begin{pmatrix}(Q^\top(\theta-\theta'))_i\\(Q^\top(\theta-\theta'-y+y'))_i\end{pmatrix}.
\]
}
The map $(\theta-\theta',y-y')\mapsto(Q^\top(\theta-\theta'),Q^\top(\theta-\theta'-y+y'))$ is invertible, so $\dist_r$ is an equivalent norm distance. Under synchronous coupling, \eqref{eq:curvature-adapted-metric} gives
{
\begin{align*}
    \mathbb E\,\dist_r(\Phi_{x_0}(z),\Phi_{x_0}(z'))^2
    &=r^2\dist_r(z,z')^2
        -\sum_i\left(
      |(Q^\top(\theta-\theta'))_i|^2
      +|(Q^\top(\theta-\theta'-y+y'))_i|^2\right)\\
    &\le r^2\dist_r(z,z')^2.
\end{align*}
}
Taking square roots proves the claimed $L^2$ GMC bound with factor $r$.

To apply the stationary construction in Appendix~\ref{app:stationary-construction}, we also need a finite second moment of the one-step displacement. At $z^\star=(\theta^\star,0)$, the lag displacement is $(-\gamma Q^\top\varepsilon,0)$. Hence
{
\begin{equation}
    D_2:=\|\dist_r(\Phi_{x_0}(z^\star),z^\star)\|_2
    =\gamma\left(\sum_i(P_i)_{11}\mathbb E(Q^\top\varepsilon)_i^2\right)^{1/2}<\infty.
    \label{eq:curvature-displacement}
\end{equation}
}

Appendix~\ref{app:stationary-construction} now applies with $p=2$, distance $\dist_r$, and rate $r$. It gives a unique invariant probability law and a causal stationary solution with
{
\[
    \|\dist_r(Z_0^\circ,z^\star)\|_2\le\frac{D_2}{1-r}.
\]
}
Norm equivalence gives its finite second moment.
\end{proof}

\begin{lemma}[Scalar rate and stability criteria]
\label{lem:curvature-rate-criterion}
Let $h>0$, $v\ge0$, $0\le\beta<1$, and $\gamma>0$. For a prescribed $L^2$ rate $0<r<1$, $\rho(\mathcal T(h,v))<r^2$ if and only if $\beta<r^2$ and
\begin{equation}
    \gamma^2v<\frac{r^2-\beta}{r^2+\beta}\left(\frac{(r^2+\beta)^2}{r^2}-(1+\beta-\gamma h)^2\right).
    \label{eq:curvature-target-rate}
\end{equation}
In particular, $\rho(\mathcal T(h,v))<1$ if and only if
\begin{equation}
    0<\gamma<\gamma_{\rm rc}(\beta;h,v):=\frac{2h(1-\beta^2)}{(1-\beta)h^2+(1+\beta)v}.
    \label{eq:curvature-stability}
\end{equation}
\end{lemma}

\begin{proof}
We obtain the condition for a prescribed $L^2$ rate by applying a known stability criterion to a rescaled recursion. Recall $a=1+\beta-\gamma h$ and $w=\gamma^2v$.
The scalar mean-square stability criterion of \citet[Appendix~B, equation~(19)]{andreyev2026momentum}, written in terms of the coefficient mean $a$ and variance $w$, is equivalent to
\[
    w<\frac{1-\beta}{1+\beta}((1+\beta)^2-a^2).
\]
This inequality includes stability of the mean recursion, since it implies $|a|<1+\beta$.

Fix $0<r\le1$. We test $\rho(\mathcal T(h,v))<r^2$ by checking mean-square stability after rescaling the difference by $r^{-k}$. Since $\det\mathcal T(h,v)=\beta^3$, we have $\rho(\mathcal T(h,v))\ge\beta$; hence $\beta<r^2$ is necessary. Assume $\beta<r^2$ and set $\widetilde e_k=r^{-k}\Delta e_{i,k}$. The difference recursion gives
\begin{align*}
    \widetilde e_k
    &=r^{-k}\left((1+\beta-\gamma H_{i,k})\Delta e_{i,k-1}-\beta\Delta e_{i,k-2}\right)\\
    &=\frac{1+\beta-\gamma H_{i,k}}r\,\widetilde e_{k-1}
      -\frac\beta{r^2}\,\widetilde e_{k-2}.
\end{align*}
To compare its second moments with those of the original difference, note that
\[
    \mathbb E(\widetilde e_k^2+\widetilde e_{k-1}^2)
    =r^{-2k}\mathbb E((\Delta e_{i,k})^2+r^2(\Delta e_{i,k-1})^2).
\]
The fixed weight $r^2$ does not change the exponential factor. Thus the rescaled recursion has second-moment factor $\rho(\mathcal T(h,v))/r^2$ and is mean-square stable exactly when $\rho(\mathcal T(h,v))<r^2$.

Its momentum coefficient, coefficient mean, and coefficient variance are $\beta/r^2$, $a/r$, and $w/r^2$. Applying the scalar stability criterion to these coefficients gives
\begin{align*}
    \frac w{r^2}
    <\frac{1-\beta/r^2}{1+\beta/r^2}
      \left((1+\beta/r^2)^2-\frac{a^2}{r^2}\right)
    \Longleftrightarrow\quad
    w<\frac{r^2-\beta}{r^2+\beta}
      \left(\frac{(r^2+\beta)^2}{r^2}-a^2\right).
\end{align*}
Together with $\beta<r^2$, this is equivalent to $\rho(\mathcal T(h,v))<r^2$.
Substituting $a=1+\beta-\gamma h$ and $w=\gamma^2v$ gives~\eqref{eq:curvature-target-rate}. For $r=1$,
\begin{align*}
    \frac{1-\beta}{1+\beta}((1+\beta)^2-a^2)-w
    &=\frac{1-\beta}{1+\beta}((1+\beta)^2-(1+\beta-\gamma h)^2)-\gamma^2v\\
    &=\frac{\gamma}{1+\beta}
      \left(2h(1-\beta^2)-\gamma((1-\beta)h^2+(1+\beta)v)\right).
\end{align*}
Its positivity is equivalent to~\eqref{eq:curvature-stability}.
\end{proof}

\subsection{Proof of Proposition~\ref{prop:curvature-variance}}
\label{app:curvature-variance}

\begin{proof}
\textit{1. Stationary variance.} Proposition~\ref{thm:curvature-rate} gives a causal stationary state with finite second moment. 
Fix a coordinate \(i\) and write $e_k^\circ=(Q^\top(\theta_k^\circ-\theta^\star))_i$.
We suppress the coordinate index below.
Let $a=1+\beta-\gamma h$ and $\zeta_k=(H_k-h)e_{k-1}^\circ+\varepsilon_k$, such that the recursion is $e_k^\circ=ae_{k-1}^\circ-\beta e_{k-2}^\circ-\gamma\zeta_k$.
We analyze variance and covariance $V=\mathbb E(e_k^\circ)^2$, and $C=\mathbb E(e_k^\circ e_{k-1}^\circ)$.

The equations for $C$ and $V$ require the conditional mean and second moment of $\zeta_k$. Since $(H_k,\varepsilon_k)$ is independent of $\mathcal F_{k-1}$,
\eqs*{
    \mathbb E(\zeta_k\mid\mathcal F_{k-1})
    =(\mathbb EH_k-h)e_{k-1}^\circ+\mathbb E\varepsilon_k=0.
}
Taking expectations in the recursion and using stationarity gives
\eqs*{
    \mathbb Ee_k^\circ=(a-\beta)\mathbb Ee_k^\circ=(1-\gamma h)\mathbb Ee_k^\circ.
}
Since $\gamma h>0$, we obtain $\mathbb Ee_k^\circ=0$.

We now expand the second moment. Independence from the past and the zero stationary mean give
\eqs*{
    \mathbb E\zeta_k^2
    &=\mathbb E((H_k-h)^2(e_{k-1}^\circ)^2)
      +2\mathbb E((H_k-h)\varepsilon_ke_{k-1}^\circ)+\sigma^2\\
    &=vV+2\mathbb E((H_k-h)\varepsilon_k)\,\mathbb Ee_{k-1}^\circ+\sigma^2\\
    &=vV+\sigma^2.
}
This calculation allows $H_k$ and $\varepsilon_k$ to be dependent.

For the covariance equation, multiply the recursion by $e_{k-1}^\circ$. For the variance equation, square $e_k^\circ$. The noise cross terms vanish since
$\mathbb E(\zeta_ke_{k-j}^\circ)
    =\mathbb E(e_{k-j}^\circ\mathbb E(\zeta_k\mid\mathcal F_{k-1}))=0$ for $j \ge 1$. 
Hence we obtain
\begin{align*}
    C&=aV-\beta C,\\
    V&=(a^2+\beta^2)V-2a\beta C+\gamma^2(vV+\sigma^2).
\end{align*}
Substitution of $C=aV/(1+\beta)$ yields
\begin{align*}
    \gamma^2\sigma^2
    &=\left(1-\beta^2-a^2+\frac{2a^2\beta}{1+\beta}-\gamma^2v\right)V\\
    &=\left(\frac{1-\beta}{1+\beta}((1+\beta)^2-a^2)-\gamma^2v\right)V\\
    &=\frac{\gamma}{1+\beta}
      \left(2h(1-\beta^2)-\gamma((1-\beta)h^2+(1+\beta)v)\right)V.
\end{align*}
The denominator below is positive by~\eqref{eq:curvature-stability}, and hence
\[
    V=\frac{\gamma\sigma^2(1+\beta)}
    {2h(1-\beta^2)-\gamma((1-\beta)h^2+(1+\beta)v)}.
\]
This proves the stationary-variance formula in~\eqref{eq:curvature-variances}.

\textit{2. Averaged limiting variance.} We express the partial sum as a martingale sum and boundary terms. Rewrite the recursion as
$\gamma he_k^\circ=-\gamma\zeta_{k+1}+\beta(e_k^\circ-e_{k-1}^\circ)-(e_{k+1}^\circ-e_k^\circ)$.
Summing over $k=1,\ldots,n$ gives
\begin{equation}
    h\sum_{k=1}^ne_k^\circ
    =-\sum_{k=1}^n\zeta_{k+1}
    +\frac{\beta(e_n^\circ-e_0^\circ)-(e_{n+1}^\circ-e_1^\circ)}{\gamma}.
    \label{eq:curvature-average-identity}
\end{equation}
After dividing~\eqref{eq:curvature-average-identity} by $\sqrt n$, stationarity gives the boundary bound
\[
    \left\|\frac{\beta(e_n^\circ-e_0^\circ)-(e_{n+1}^\circ-e_1^\circ)}{\gamma\sqrt n}\right\|_2
    \le\frac{2(1+\beta)\sqrt V}{\gamma\sqrt n}\to0.
\]
The martingale differences are orthogonal, so
\[
    \frac1n\mathbb E\left(\sum_{k=1}^n\zeta_{k+1}\right)^2
    =\frac1n\sum_{k=1}^n\mathbb E\zeta_{k+1}^2=vV+\sigma^2.
\]
Combining these two bounds gives
\[
    \lim_{n\to\infty}\frac1n\operatorname{Var}\left(\sum_{k=1}^ne_k^\circ\right)
    =\frac{vV+\sigma^2}{h^2}.
\]

By Proposition~\ref{prop:physical-dependence}, the stationary covariance series is absolutely summable, so the preceding variance limit equals $\Sigma_i^{\rm av}$, giving the second formula in~\eqref{eq:curvature-variances}. Theorem~\ref{thm:fclt} then yields the CLT from every deterministic initialization.
\end{proof}

\subsection{Proof of Theorem~\ref{thm:curvature-acceleration}}
\label{app:rate-variance-proof}

Here $\kappa=L/\mu\ge4$, as in the theorem.

\begin{proof}
For part~\textup{(i)}, the same tuning must give fast forgetting in the $\mu$ direction and mean-square stability in the $L$ direction. The former requires a sufficiently large normalized step $\eta=\gamma/(1-\beta)$, while random curvature in the latter gives an upper bound on this step. Combining these requirements bounds $1-q$ as in~\eqref{eq:curvature-noise-necessity}. The same lower bound on $\eta$ also gives the averaged limiting variance bound~\eqref{eq:rate-variance-lower}. The mean lag matrix is
\[
    \bar B_\mu=\begin{pmatrix}1+\beta-\gamma\mu&-\beta\\1&0\end{pmatrix}.
\]
Independence of the innovations and Jensen's inequality give
\[
    |\bar B_\mu^{\,k}u|^2
    =|\mathbb E(B_{\mu,k}\cdots B_{\mu,1}u)|^2
    \le\mathbb E|B_{\mu,k}\cdots B_{\mu,1}u|^2.
\]
Taking the supremum over $|u|=1$, then $k$th roots, and letting $k\to\infty$ gives
\[
    \rho(\bar B_\mu)^2\le R_\mu\le R_{\max}=q^2,\qquad
    \beta=(\det\mathcal T(\mu,v_\mu))^{1/3}\le R_\mu\le q^2.
\]
Thus $\rho(\bar B_\mu) \le q$ and $\beta \le q^2$.
Here $q>0$. Otherwise $\beta=0$ and $R_\mu=R_L=0$ would imply $\gamma\mu=\gamma L=1$, contrary to $\mu<L$.

Since \(\rho(\bar B_\mu)\le q\), every eigenvalue of \(\bar B_\mu\) has modulus at most \(q\). Hence the roots of its characteristic polynomial
\[
p(z)=z^2-(1+\beta-\gamma\mu)z+\beta
\]
have modulus at most \(q\).
For real roots, both factors in $p(q)$ are nonnegative; for nonreal roots, $p(q)$ is a squared modulus. In either case $p(q)\ge0$, which gives $\gamma\mu\ge(1-q)(1-\beta/q)$.
Also, note that $\frac{1-\beta/q}{1-\beta}-\frac1{1+q} =\frac{q^2-\beta}{q(1-\beta)(1+q)}\ge0$ by $\beta \le q^2$.
Combining these gives the required lower bound on $\eta L$:
\[
    \eta L=\frac{\kappa\gamma\mu}{1-\beta}
    \ge\kappa(1-q)\frac{1-\beta/q}{1-\beta}
    \ge\frac{\kappa(1-q)}{1+q}.
\]
For the upper bound, apply~\eqref{eq:curvature-stability} to the $L$ coordinate, where $v_L=\nu L^2$:
\[
    \eta L<\frac{2}{\nu+(1-\beta)/(1+\beta)}<\frac2\nu.
\]
Combining the lower and upper bounds gives
\[
    \nu\kappa(1-q)<2(1+q)=4-2(1-q),\qquad
    1-q\le\frac4{\nu\kappa+2},
\]
which proves~\eqref{eq:curvature-noise-necessity}.

To prove~\eqref{eq:rate-variance-lower}, use the same lower bound on $\eta L$. If $\sigma_L^2>0$, the variance formula gives
\[
    V_L=\frac{\gamma\sigma_L^2}
    {2L(1-\beta)-\gamma L^2((1-\beta)/(1+\beta)+\nu)}
    \ge\frac{\gamma\sigma_L^2}{2L(1-\beta)}.
\]
Hence
\[
    \frac{\Sigma_L^{\rm av}}{\sigma_L^2/L^2}-1=\frac{\nu L^2V_L}{\sigma_L^2}
    \ge\frac{\nu\eta L}{2}
    \ge\frac{\nu\kappa}{2}\frac{1-q}{1+q}.
\]
Adding one proves~\eqref{eq:rate-variance-lower}. If $q\le1-a/\sqrt\kappa$, then
\[
    \frac{\Sigma_L^{\rm av}}{\sigma_L^2/L^2}-1\ge\frac{\nu\kappa}{2}\frac{a/\sqrt\kappa}{2}
    =\frac a4\nu\sqrt\kappa.
\]

For part~\textup{(ii)}, let $s=\kappa^{-1/2}$ and $\beta=(1-s)^2$. We first explain the learning-rate choice using the fixed-curvature model. Since $\beta>\beta_\star=((1-s)/(1+s))^2$, Corollary~\ref{cor:fixed-beta-optimum} gives $\rho_\star(\beta)=\sqrt\beta=1-s$. By Theorem~\ref{thm:exact-quadratic-rate}, this rate is attained when $|1+\beta-\gamma\lambda|\le2\sqrt\beta$ for every $\lambda\in[\mu,L]$. Using $\mu=s^2L$, this condition becomes
\eqs*{
    \frac{(1-\sqrt\beta)^2}{\mu}\le\gamma\le\frac{(1+\sqrt\beta)^2}{L}
    \quad\Longleftrightarrow\quad
    \frac1L\le\gamma\le\frac{(2-s)^2}{L}.
}
We choose the smallest learning rate in this interval, $\gamma_{\rm SGDM}=1/L$. This also minimizes the curvature-noise term $\gamma^2v_i$ over this interval.

We next show that this tuning gives $q<1-s/2$ under the stated curvature-variance bound. Let $r=1-s/2$ and $t=\lambda_i/L\in[s^2,1]$. Since $q^2=\max_iR_i$, it suffices to prove $R_i<r^2$ for every $i$. By Lemma~\ref{lem:curvature-rate-criterion}, we need to check
\[
    \beta<r^2,\qquad
    \gamma_{\rm SGDM}^2v_i<
    \frac{r^2-\beta}{r^2+\beta}
    \left(\frac{(r^2+\beta)^2}{r^2}-(1+\beta-t)^2\right).
\]

We bound the left side from above and the right side from below. Write $D(t)=(r^2+\beta)^2/r^2-(1+\beta-t)^2$ for the last factor. Since $0<s\le1/2$,
\[
    r^2-\beta=s-\frac{3s^2}{4}\ge\frac{5s}{8},\qquad
    r^2+\beta\le2,\qquad \frac{r^2-\beta}{r^2+\beta}\ge\frac{5s}{16}.
\]
At $t=s^2$, we have $1+\beta-s^2=2(1-s)=2\sqrt\beta$, and therefore
\[
    D(s^2)=\frac{(r^2+\beta)^2-4\beta r^2}{r^2}
    =\frac{(r^2-\beta)^2}{r^2}\ge\frac{25s^2}{64}.
\]
For $s^2\le t\le1$,
\begin{align*}
    D(t)-D(s^2)&=(t-s^2)(2(1+\beta)-t-s^2),\\
    2(1+\beta)-t-s^2&\ge3-4s+s^2\ge1,\\
    D(t)&\ge\frac{25s^2}{64}+t-s^2
    =\frac{25t}{64}+\frac{39}{64}(t-s^2)\ge\frac{25t}{64}.
\end{align*}
Since $\gamma_{\rm SGDM}=1/L$ and $v_i\le\nu\lambda_i^2$,
\begin{align*}
    \gamma_{\rm SGDM}^2v_i \le\nu t^2\le\frac{s}{16}t^2\le\frac{st}{16},\qquad
    \frac{r^2-\beta}{r^2+\beta}D(t)\ge\frac{5s}{16}\frac{25t}{64}>\frac{st}{16}.
\end{align*}
Lemma~\ref{lem:curvature-rate-criterion} gives $R_i<r^2$ for all $i$. Also $\det\mathcal T(\lambda_i,v_i)=\beta^3$ implies $R_i\ge \beta$. Thus
\[
    \beta\le R_{\max}<r^2,\qquad 1-s\le q<1-s/2=\bar q.
\]

For $\sigma_i^2>0$, it remains to bound the averaged limiting variance. By~\eqref{eq:curvature-variances},
\begin{align*}
    \frac{\Sigma_i^{\rm av}}{\sigma_i^2/h_i^2}-1
    &=\frac{v_iV_i}{\sigma_i^2}\\
    &=\frac{v_i\gamma(1+\beta)}
      {2h_i(1-\beta^2)-\gamma((1-\beta)h_i^2+(1+\beta)v_i)}\\
    &=\frac{(v_i/h_i^2)t}
      {(1-\beta)(2-t/(1+\beta))-(v_i/h_i^2)t}.
\end{align*}
Since $t=h_i/L\le1$ and $v_i/h_i^2\le\nu\le s/16$, the denominator satisfies
\[
    (1-\beta)\left(2-\frac{t}{1+\beta}\right)-\frac{v_i}{h_i^2}t
    \ge s(2-s)-\frac{s}{16}
    \ge\frac{3s}{2}-\frac{s}{16}\ge s.
\]
Therefore,
\[
    0\le\frac{\Sigma_i^{\rm av}}{\sigma_i^2/h_i^2}-1
    \le\frac{v_i t}{h_i^2s}\le\frac\nu s=\nu\sqrt\kappa.
\]
This proves~\eqref{eq:rate-variance-upper}.
\end{proof}

\subsection{Proof of Theorem~\ref{thm:rv-finite-horizon}}
\label{app:rv-finite-horizon}

Use $\bar\theta_{b,n}$, $\mathcal S$, and $\mathcal D_0$ from Section~\ref{sec:quadratic-finite-averaging}. Couple the given trajectory to a stationary trajectory using the same future innovations. We bound the two terms in
\[
    \|\bar\theta_{b,n}-\theta^\star\|_2
    \le\|\bar\theta_{b,n}^\circ-\theta^\star\|_2
      +\|\bar\theta_{b,n}-\bar\theta_{b,n}^\circ\|_2.
\]
The first term comes from the stationary covariances. The second is the effect of initialization. The GMC bound above gives the decay rate but does not control how its constant depends on $\kappa$. To track the $\kappa$-dependence of the finite-horizon error, we first prove a coupling bound with a constant independent of $\kappa$. We use this bound for both the stationary covariance tail and the initialization effect. Throughout this subsection, use the tuning~\eqref{eq:curvature-retuning} and write $s=\kappa^{-1/2}$, $\bar q=1-s/2$, and $r=\bar q$.

\begin{lemma}[Uniform coupling bound for the parameter iterates]
\label{lem:rv-transient}
Under~\eqref{eq:curvature-retuning} and $\nu\le s/16$, two synchronously coupled trajectories satisfy, coordinatewise,
\begin{equation}
    \mathbb E|\Delta e_{i,k}|^2
    \le66\bar q^{\,2k}\mathbb E\left(|\Delta e_{i,0}|^2+
    \frac{L}{h_i}|\Delta y_{i,0}|^2\right).
    \label{eq:rv-transient}
\end{equation}
Here $\Delta y_{i,0}$ is the initial momentum difference in the $Q$ basis. Random initial states are allowed if they are square-integrable and independent of the future innovations. The constant is uniform in $\kappa$; the weight $L/h_i$ on the initial momentum is explicit.
\end{lemma}

\begin{proof}
Fix $i$ and put $t=h_i/L$, $a=1+\beta-t$, and $w=\gamma^2v_i$. Use the lag update $B_{i,k}$ from Appendix~\ref{app:curvature-rate}. We choose a squared norm $u^\top P_r u$ that contracts by $r^2$, then compare it with the parameter difference.

Set $(P_r)_{11}=1$. We reduce contraction to one scalar inequality by setting the off-diagonal and lower-right entries of $r^2P_r-\mathbb E(B_{i,k}^\top P_rB_{i,k})$ to zero:
\[
    (r^2+\beta)(P_r)_{12}+a\beta=0,\qquad r^2(P_r)_{22}-\beta^2=0.
\]
This gives
\[
    P_r=\begin{pmatrix}1&-\dfrac{a\beta}{r^2+\beta}\\[2mm]-\dfrac{a\beta}{r^2+\beta}&\dfrac{\beta^2}{r^2}\end{pmatrix}.
\]
Recall $D(t)=(r^2+\beta)^2/r^2-a^2$ from Appendix~\ref{app:rate-variance-proof}. Direct multiplication gives
\begin{align*}
    r^2P_r-\mathbb E(B_{i,k}^\top P_rB_{i,k})
    &=\left(r^2-a^2-w+\frac{2a^2\beta}{r^2+\beta}-\frac{\beta^2}{r^2}\right)
      \begin{pmatrix}1&0\\0&0\end{pmatrix}\\
    &=\left(\frac{r^2-\beta}{r^2+\beta}D(t)-w\right)
      \begin{pmatrix}1&0\\0&0\end{pmatrix}\succeq0.
\end{align*}
The last step uses the bound on $w$ proved in Appendix~\ref{app:rate-variance-proof}. That proof also gives $\beta<r^2$ and $D(t)\ge25t/64$, so
\[
    \det P_r=\frac{\beta^2D(t)}{(r^2+\beta)^2}>0,\qquad P_r\succ0.
\]
Iterating the conditional contraction yields
\begin{align*}
    \mathbb E\left((\Delta e_{i,k},\Delta e_{i,k-1})P_r
       \begin{pmatrix}\Delta e_{i,k}\\\Delta e_{i,k-1}\end{pmatrix}\right)
       \le r^{2k}\mathbb E\left((\Delta e_{i,0},\Delta e_{i,-1})P_r
       \begin{pmatrix}\Delta e_{i,0}\\\Delta e_{i,-1}\end{pmatrix}\right).
\end{align*}

For parameter difference $x$ and momentum difference $y$, the lag difference is $(x,x-y)^\top$. We need a lower bound in $t x^2$ at time $k$ and an upper bound in $t x^2+y^2$ at time zero. Both constants must be independent of $\kappa$.

For the lower bound, complete the square:
\begin{align*}
    (x,x-y)P_r\begin{pmatrix}x\\x-y\end{pmatrix}
    &=\frac{\beta^2}{r^2}\left(x-y-\frac{ar^2}{\beta(r^2+\beta)}x\right)^2
      +\frac{r^2D(t)}{(r^2+\beta)^2}x^2
      \ge\frac{25}{256}t x^2,
\end{align*}
where $r^2/(r^2+\beta)^2\ge1/4$ and $D(t)\ge25t/64$.

For the upper bound, first write
\[
    \begin{pmatrix}x\\x-y\end{pmatrix}
    =x\begin{pmatrix}1\\1\end{pmatrix}
      -y\begin{pmatrix}0\\1\end{pmatrix}.
\]
Since $P_r\succ0$,
\[
    (x,x-y)P_r\begin{pmatrix}x\\x-y\end{pmatrix}
    \le2x^2(1,1)P_r(1,1)^\top+2y^2(P_r)_{22}.
\]
It is enough to bound $(1,1)P_r(1,1)^\top$ by $3t$ and $(P_r)_{22}$ by $1$. The second bound follows from $(P_r)_{22}=\beta^2/r^2\le1$. For the first, use $\beta=(1-s)^2$ and
\[
    1-\frac{\beta}{\bar q}=\frac{s(3/2-s)}{1-s/2}\le\frac32s,\qquad
    \bar q-\sqrt\beta=\frac s2,\qquad \bar q\ge\frac34.
\]
Since $a=2\sqrt\beta-(t-s^2)$,
\begin{align*}
    (1,1)P_r(1,1)^\top
    &=1-\frac{4\beta^{3/2}}{r^2+\beta}+\frac{\beta^2}{r^2}
      +\frac{2\beta}{r^2+\beta}(t-s^2)\\
    &=\left(1-\frac{\beta}{\bar q}\right)^2
      +\frac{2\beta}{r^2+\beta}\frac{(\bar q-\sqrt\beta)^2}{\bar q}
      +\frac{2\beta}{r^2+\beta}(t-s^2)\\
    &\le\left(\frac94+\frac13\right)s^2+(t-s^2)\le3t.
\end{align*}
Substituting these two bounds gives
\[
    (x,x-y)P_r\begin{pmatrix}x\\x-y\end{pmatrix}\le6t x^2+2y^2.
\]
Now combine the norm comparisons with the iterated contraction. Since $\Delta e_{i,-1}=\Delta e_{i,0}-\Delta y_{i,0}$,
\begin{align*}
    \mathbb E|\Delta e_{i,k}|^2
    &\le\frac{256}{25t}r^{2k}
      \mathbb E\left(6t|\Delta e_{i,0}|^2+2|\Delta y_{i,0}|^2\right)\\
    &\le\frac{1536}{25}r^{2k}
      \mathbb E\left(|\Delta e_{i,0}|^2+\frac{|\Delta y_{i,0}|^2}{t}\right)\\
    &\le66\bar q^{\,2k}
      \mathbb E\left(|\Delta e_{i,0}|^2+\frac{L}{h_i}|\Delta y_{i,0}|^2\right).
\end{align*}
This proves~\eqref{eq:rv-transient}.
\end{proof}

\begin{proof}[Proof of Theorem~\ref{thm:rv-finite-horizon}]
\textit{1. Stationary averaging error.}
We first compare the variance of the stationary average with $\mathcal S/n$. Fix a coordinate and write $t=h/L$, $a=1+\beta-t$, and $C_j=\mathbb E(e_0^\circ e_j^\circ)$. Conditioning on the past gives
\[
    C_0=V,\qquad C_1=\frac{aV}{1+\beta},\qquad C_{j+1}=aC_j-\beta C_{j-1}\quad(j\ge1).
\]
GMC gives geometric covariance decay, so the weighted sums below converge absolutely. Summing the recurrence gives
\begin{align*}
    \sum_{j\ge1}C_j-C_1
    &=a\sum_{j\ge1}C_j-\beta\left(V+\sum_{j\ge1}C_j\right),\\
    t\sum_{j\ge1}C_j&=C_1-\beta V.
\end{align*}
Hence
\begin{align*}
    \Sigma^{\rm av}
    &=V+2\sum_{j\ge1}C_j
      =V+\frac2t\left(\frac{aV}{1+\beta}-\beta V\right)\\
    &=V(1-\beta)\left(\frac2t-\frac1{1+\beta}\right)
      \ge\frac{sV}{t}.
\end{align*}
The last inequality uses $1-\beta\ge s$ and $2-t/(1+\beta)\ge1$.

The finite-average error depends on a weighted covariance tail. Write $J_n=\sum_{\ell\ge1}\ell C_{n+\ell}$. Then
\begin{align*}
    \mathbb E\left(\frac1n\sum_{k=1}^ne_k^\circ\right)^2
    &=\frac Vn+\frac2{n^2}\sum_{j=1}^{n-1}(n-j)C_j\\
    &=\frac{\Sigma^{\rm av}}n
      -\frac2{n^2}\left(\sum_{j=1}^{n-1}jC_j+n\sum_{j=n}^\infty C_j\right)\\
    &=\frac{\Sigma^{\rm av}}n-\frac{2(J_0-J_n)}{n^2}.
\end{align*}
We therefore need a bound on $J_0-J_n$. Multiplying the covariance recurrence by $j$ and summing gives
\begin{align*}
    J_0-\sum_{j\ge1}C_j
    &=aJ_0-\beta\left(V+J_0+\sum_{j\ge1}C_j\right),\\
    tJ_0&=(1-\beta)\sum_{j\ge1}C_j-\beta V
      =\frac{(1-\beta)\Sigma^{\rm av}-(1+\beta)V}{2}.
\end{align*}
Also,
\[
    J_1-J_0=-\frac{\Sigma^{\rm av}-V}{2},\qquad
    J_{n+1}=aJ_n-\beta J_{n-1}\quad(n\ge1).
\]
This is the same deterministic recursion covered by Lemma~\ref{lem:rv-transient}. We first bound its initial values. Since $s^2\le t\le1$, $1-\beta\le2s$, and $V\le t\Sigma^{\rm av}/s$,
\begin{align*}
    |J_0|
    &\le\frac{\max((1-\beta)\Sigma^{\rm av},(1+\beta)V)}{2t}
      \le\max\left(\frac{s\Sigma^{\rm av}}t,\frac Vt\right)
      \le\frac{\Sigma^{\rm av}}s,\\
    \frac{|J_1-J_0|}{\sqrt t}
    &\le\frac{\max(\Sigma^{\rm av},V)}{2\sqrt t}
      \le\frac{\Sigma^{\rm av}}{2s},\\
    |J_1|&\le|J_0|+|J_1-J_0|
      \le\frac{3\Sigma^{\rm av}}{2s}.
\end{align*}
Apply Lemma~\ref{lem:rv-transient} with $v=0$, starting at time $1$, with parameter $J_1$ and momentum $J_1-J_0$. For $n\ge1$,
\[
    |J_n|^2\le66r^{2(n-1)}\left(J_1^2+\frac{(J_1-J_0)^2}{t}\right)
    \le165\left(\frac{\Sigma^{\rm av}}s\right)^2.
\]
Thus
\[
    2|J_0-J_n|\le2(|J_0|+|J_n|)
    \le2(1+\sqrt{165})\frac{\Sigma^{\rm av}}s
    \le\frac{28\Sigma^{\rm av}}s.
\]
Insert this into the finite-average identity and sum over coordinates:
\begin{equation}
    \left|\mathbb E|\bar\theta_{b,n}^\circ-\theta^\star|^2-\frac{\mathcal S}{n}\right|
    \le\sum_i\frac{28\Sigma_i^{\rm av}}{sn^2}
    =\frac{28\sqrt\kappa\,\mathcal S}{n^2}.
    \label{eq:rv-stationary-remainder-explicit}
\end{equation}

\textit{2. Initialization error.}
We now bound the difference between the two averages. First compute the initial difference needed in Lemma~\ref{lem:rv-transient}. With $t_i=h_i/L$, the stationary momentum satisfies
\[
    \mathbb E(y_{i,0}^\circ)^2
    =2V_i\left(1-\frac{1+\beta-t_i}{1+\beta}\right)
    =\frac{2t_iV_i}{1+\beta}.
\]
Here $y_{i,0}^\circ=e_{i,0}^\circ-e_{i,-1}^\circ$ is the momentum in the $Q$ basis. The stationary error and momentum have mean zero. Therefore,
{
\begin{align*}
    \sum_i\mathbb E\left(|\Delta e_{i,0}|^2+\frac{|\Delta y_{i,0}|^2}{t_i}\right)
    &=|\theta_0-\theta^\star|^2+\sum_iV_i
      +\sum_i\frac{\bigl((Q^\top y_0)_i\bigr)^2}{t_i}+\sum_i\frac{\mathbb E(y_{i,0}^\circ)^2}{t_i}\\
    &=|\theta_0-\theta^\star|^2+\sum_i\frac{L}{h_i}\bigl((Q^\top y_0)_i\bigr)^2
      +\left(1+\frac2{1+\beta}\right)\sum_iV_i
      =\mathcal D_0.
\end{align*}
}
Lemma~\ref{lem:rv-transient} and Minkowski's inequality give
\begin{align*}
    \|\bar\theta_{b,n}-\bar\theta_{b,n}^\circ\|_2
    &\le\frac1n\sum_{k=b+1}^{b+n}\|\theta_k-\theta_k^\circ\|_2\\
    &\le\frac{\sqrt{66\mathcal D_0}}n\sum_{k=b+1}^{b+n}\bar q^k
      =\sqrt{66\mathcal D_0}\,\bar q^{b+1}
        \frac{1-\bar q^n}{n(1-\bar q)}.
\end{align*}
Since $1-\bar q=s/2$ and $\bar q\le e^{-s/2}$,
\[
    \|\bar\theta_{b,n}-\bar\theta_{b,n}^\circ\|_2^2
    \le\frac{66\mathcal D_0}{n^2(s/2)^2}e^{-s(b+1)}
    =\frac{264\kappa\mathcal D_0}{n^2}
      \exp\left(-\frac{b+1}{\sqrt\kappa}\right).
\]

\textit{3. Combine the two bounds.}
Combining the two bounds by the triangle inequality gives
\[
    \|\bar\theta_{b,n}-\theta^\star\|_2
    \le\left(\frac{\mathcal S}{n}\left(1+\frac{28\sqrt\kappa}{n}\right)\right)^{1/2}
      +\frac{\sqrt{264}\sqrt{\kappa\mathcal D_0}}n
       \exp\left(-\frac{b+1}{2\sqrt\kappa}\right).
\]
Since $\sqrt{264}<28$, the constant $C=28$ suffices for both bounds in Theorem~\ref{thm:rv-finite-horizon}.

When $y_0=0$, the stationary variance bound above gives
\[
    \Sigma_i^{\rm av}\ge\frac{sV_i}{t_i}\ge sV_i,\qquad
    \sum_iV_i\le\frac{\mathcal S}s.
\]
Therefore,
\[
    \mathcal D_0=|\theta_0-\theta^\star|^2
      +\left(1+\frac2{1+\beta}\right)\sum_iV_i
    \le|\theta_0-\theta^\star|^2+3\sqrt\kappa\,\mathcal S.
\]

For $\mathcal S>0$ and $0<\epsilon<1$, a sufficient choice for the relative mean-square error bound below is
\begin{equation}
    n\ge\frac{28\sqrt\kappa}{\epsilon},\qquad
    b+1\ge\sqrt\kappa\log_+\left(\frac{264\kappa\mathcal D_0}{\epsilon^2 n\mathcal S}\right),
    \label{eq:rv-variance-dominance}
\end{equation}
where $\log_+x=\max(0,\log x)$. These conditions give
\[
    \frac{28\sqrt\kappa\,\mathcal S}{n^2}\le\epsilon\frac{\mathcal S}{n},\qquad
    \frac{264\kappa\mathcal D_0}{n^2}e^{-(b+1)/\sqrt\kappa}
    \le\epsilon^2\frac{\mathcal S}{n}.
\]
Hence the same triangle inequality yields
\[
    \mathbb E|\bar\theta_{b,n}-\theta^\star|^2
    \le(\sqrt{1+\epsilon}+\epsilon)^2\frac{\mathcal S}{n}.
\]
\end{proof}

\subsection{Coordinate-sampled least squares and iid minibatches}
\label{app:coordinate-batching}
Let $d\ge2$ and let $Q$ be orthogonal. Fix $\lambda_i\in[\mu,L]$ with both endpoints attained. Independently at each sample, draw $I$ uniformly from $\{1,\ldots,d\}$ and set
\[
    a=\sqrt{d\lambda_I}Qe_I,\qquad Y=a^\top\theta^\star+\xi,\qquad \mathbb E\xi=0,\quad \mathbb E\xi^2=\tau^2<\infty,
\]
where $\xi$ is independent of $I$. The sample loss has gradient
\[
    \nabla_\theta\frac{(a^\top\theta-Y)^2}{2}
    =a(a^\top\theta-Y)
    =aa^\top(\theta-\theta^\star)-a\xi.
\]
At each iteration average $B$ iid samples, drawn with replacement, independently of the past. Then
\[
    H=\frac1B\sum_{b=1}^Ba_ba_b^\top,\qquad \varepsilon=-\frac1B\sum_{b=1}^Ba_b\xi_b.
\]
Let $N_i=\sum_{b=1}^B\mathbf1_{\{I_b=i\}}$. The multinomial count moments are
\[
    \mathbb EN_i=\frac Bd,\qquad
    \operatorname{Var}(N_i)=\frac Bd\left(1-\frac1d\right),\qquad
    \operatorname{Cov}(N_i,N_j)=-\frac B{d^2}\quad(i\ne j).
\]
Since $aa^\top=d\lambda_IQe_Ie_I^\top Q^\top$,
\[
    H=Q\operatorname{diag}(H_1,\ldots,H_d)Q^\top,\qquad H_i=\frac{d\lambda_iN_i}{B}.
\]
Consequently,
\begin{align*}
    \mathbb EH_i&=\frac{d\lambda_i}{B}\frac Bd=\lambda_i,\\
    \operatorname{Var}(H_i)&=\frac{d^2\lambda_i^2}{B^2}\frac Bd\left(1-\frac1d\right)
    =\frac{d-1}{B}\lambda_i^2,\\
    \operatorname{Cov}(H_i,H_j)&=\frac{d^2\lambda_i\lambda_j}{B^2}\left(-\frac B{d^2}\right)
    =-\frac{\lambda_i\lambda_j}{B}\quad(i\ne j).
\end{align*}
For the additive noise,
\[
    (Q^\top\varepsilon)_i=-\frac{\sqrt{d\lambda_i}}{B}
    \sum_{b=1}^B\mathbf1_{\{I_b=i\}}\xi_b.
\]
Independence of the samples and $\mathbb E\xi_b=0$ give
\begin{align*}
    \mathbb E(Q^\top\varepsilon)_i&=0,\\
    \mathbb E((Q^\top\varepsilon)_i^2)
    &=\frac{d\lambda_i}{B^2}\sum_{b=1}^B
      \mathbb E(\mathbf1_{\{I_b=i\}}\xi_b^2)
    =\frac{d\lambda_i}{B^2}\frac{B\tau^2}{d}
    =\frac{\tau^2\lambda_i}{B}.
\end{align*}
Thus this example satisfies~\eqref{eq:commuting-curvature-model}, with $v_i=(d-1)\lambda_i^2/B$ and $\sigma_i^2=\tau^2\lambda_i/B$. Curvature and additive noise may be dependent through the selected coordinates. Their independence is not needed in Proposition~\ref{prop:curvature-variance}.

\section{Supplementary numerical material}
\label{app:numerical}
\label{app:temporary-numerical}

Figure~\ref{fig:numerical-appendix} and Tables~\ref{tab:gmc}--\ref{tab:curvature} provide supplementary material for Figure~\ref{fig:numerical-report}.

\begingroup
\renewcommand{\thefigure}{S\arabic{figure}}
\renewcommand{\theHfigure}{S\arabic{figure}}
\setcounter{figure}{0}
\begin{figure}[!htbp]
    \centering
    \includegraphics[width= 0.8\textwidth]{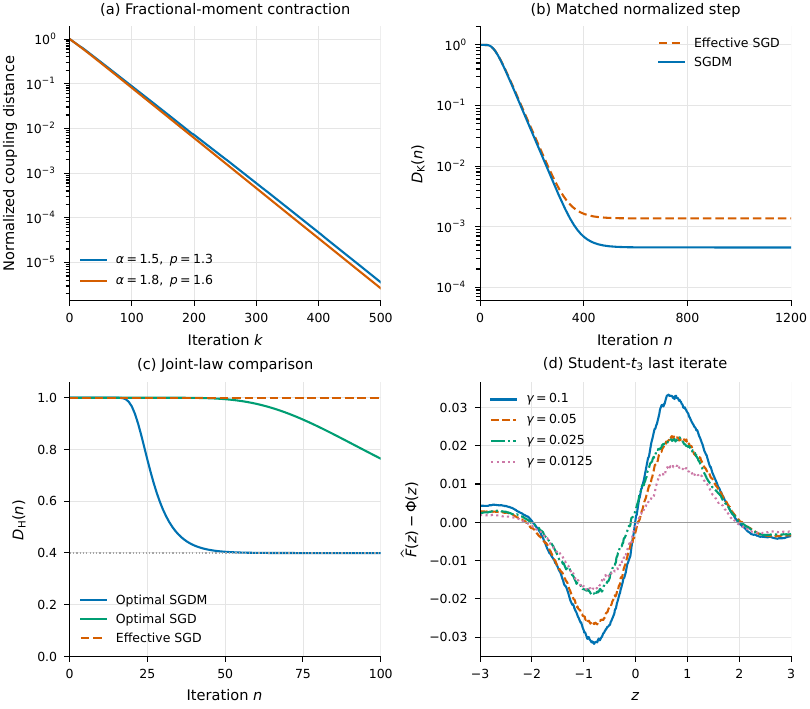}
    \caption{Supplementary diagnostics. (a) Normalized fractional-moment coupling distances for two heavy-tailed nonlinear regimes. (b) First-coordinate Kolmogorov distance to the Gaussian reference at matched normalized steps, with both methods stable. (c) Joint-law half-space distance in the quadratic accelerated setting, including unstable effective SGD. (d) Empirical CDF discrepancy for the scaled last iterate in the linear Student-$t_3$ model.}
    \label{fig:numerical-appendix}
\end{figure}
\endgroup

\subsection[Nonlinear stationarity and GMC]{Nonlinear stationarity and GMC}
\label{app:numerical-gmc}

Figure~\ref{fig:numerical-appendix}(a) and Table~\ref{tab:gmc} record the coupling diagnostic and verification of the sufficient GMC conditions corresponding to the stationary illustration in Figure~\ref{fig:numerical-report}(a).
We use the nonlinear model of Section~\ref{sec:numerical-stationarity-averaging}, with $R=0.15T_\alpha$, $\varepsilon\sim N(0,1)$, $(\beta,\gamma)=(0.5,0.01)$, and $y_0=0$. The two regimes are $(\alpha,p)=(1.5,1.3)$ and $(1.8,1.6)$, each with 50,000 paths and initial values $\theta_0=0,1,2,4$. In both regimes, curvature has infinite variance and a finite $p$th moment. Since $G'(\theta)=1+\cos\theta/4\ge3/4$, we use $\mu=3/4$ and the Lipschitz bound
\[
 L_p=\left(\mathbb E\left(|1+R|+\frac14\right)^p\right)^{1/p}.
\]

The coupling distance uses $\dist_{\mathrm{sf},p}$ in~\eqref{eq:slow-fast-metric}, with the weight $w=w_p$ reported in Table~\ref{tab:gmc}. Write $\widehat D_{k,p}$ for its empirical $L^p$ norm between paths driven by the same observations, and $D_{0,p}$ for the initial distance. Within each regime, all starts share the same observations; the two regimes are independent. For the coupling calculation, $Z_k^{(j)}$ and $Z_k^{\prime(j)}$ start at $(\theta_0,y_0)=(0,0)$ and $(2,0)$, respectively, and
\[
 \widehat D_{k,p}=\left(\frac1M\sum_{j=1}^{M}\dist_{\mathrm{sf},p}\left(Z_k^{(j)},Z_k^{\prime(j)}\right)^p\right)^{1/p},\qquad M=50000.
\]
Thus $D_{0,p}=2$. Panel (a) plots $\widehat D_{k,p}/D_{0,p}$ for $0\le k\le500$. Both rows satisfy the sufficient condition~\eqref{eq:perron-condition} and have $r<1$. The curves illustrate the corresponding coupling decay.
\begin{table}[!htbp]
\caption{Verification of the sufficient conditions in Theorem~\ref{thm:perron-gmc} for $g=(1+R)\theta+\tfrac14\sin\theta+\varepsilon$, $R=0.15\,T_\alpha$, $(\beta,\gamma)=(0.5,0.01)$, and the observed coupling distance at $k=500$.}
\label{tab:gmc}
\begin{center}
\begin{tabular}{lrrrr}
\toprule
$(\alpha,p)$ & $L_p$ & $r$ & $w_p$ & $\widehat D_{500,p}/D_{0,p}$ \\
\midrule
$(1.5,1.3)$ & 1.4636 & 0.99857 & 0.1728 & $3.61\times10^{-6}$ \\
$(1.8,1.6)$ & 1.3953 & 0.98846 & 0.4880 & $2.65\times10^{-6}$ \\
\bottomrule
\end{tabular}
\end{center}
\end{table}

\subsection[Averaged SGDM invariance principle]{Averaged SGDM invariance principle}
\label{app:numerical-invariance}

Table~\ref{tab:asgdm} reports the maximum-norm distances in Figure~\ref{fig:numerical-report}(b).
Here $R=0.5T_3$, $(\beta,\gamma)=(0.5,0.05)$, and the stationary mean is zero by symmetry. We use 20,000 paths from each of $\theta_0=-3,1,3$, with $y_0=0$, at $n=1000,3000,10000,30000,100000$. The recorded bound $L_2=1.5365$ gives $\gamma_{\rm MS}=0.1059$ in~\eqref{eq:all-beta-step}, so the chosen step satisfies the explicit second-moment condition.

An independent reference simulation uses 20,000 chains after a burn-in of 1000 steps. The reference chains start at $(\theta_0,y_0)=(0,0)$. Using the known stationary mean zero, we compute
\[
 \widehat\Sigma_T=\frac1M\sum_{j=1}^{M}\frac{S_{T,j}^2}{T},\qquad S_{T,j}=\sum_{k=1}^{T}\theta_{B+k}^{(j)},\qquad M=20000,\quad B=1000.
\]
No sample mean is subtracted. Reference lengths 25,000, 50,000, and 100,000 give estimates of the averaged limiting variance $0.6605$, $0.6688$, and $0.6599$, respectively. These estimates use different lengths of the same reference trajectories after burn-in. We use the last estimate, $\widehat\Sigma^{\rm av}=0.6599$, for all starts and horizons in the normalization of $\mathbb S_n$ defined in Section~\ref{sec:numerical-stationarity-averaging}. The test paths have no burn-in. Different starts use independent paths, while different horizons for a given start use the same paths at different horizons. For each path, its maximum is taken over the partial sums through time $n$, and the empirical CDF is compared with that of $\sup_{0\le s\le1}|B(s)|$. The Brownian CDF uses a fixed truncation to 401 terms, $-200\le k\le200$, in its normal-CDF series. A comparison with 10,001 terms changes the CDF by at most $2\times10^{-16}$ over the range of the observed maxima. Empirical Kolmogorov distances use both sides of the empirical CDF at every sorted sample value.

The simultaneous 95\% empirical-CDF error bound from the Dvoretzky--Kiefer--Wolfowitz (DKW) inequality and a union bound is $0.0126$ over the fifteen choices of initial value and horizon. It controls CDF sampling error and excludes uncertainty in $\widehat\Sigma^{\rm av}$. A recalculation with the same seed gives a sample standard error of $0.0066$ for the reference estimate at $T=100000$, computed as the sample standard deviation of $S_{T,j}^2/T$ divided by $\sqrt M$. Fourth moments are infinite, so this sample value does not give a confidence bound. A 1\% change in the reference variance has a first-order effect of at most about $0.005$ on the Brownian comparison, common to all fifteen settings. The final distances are below the DKW bound, so the calculation does not resolve their ordering. This comparison illustrates a maximum-norm consequence of the invariance principle.
\begin{table}[!htbp]
\caption{$d_K(\sup_{s\le1}|\mathbb S_n(s)|,\sup_{s\le1}|B(s)|)$ for averaged SGDM, $R=0.5\,T_3$, $(\beta,\gamma)=(0.5,0.05)$, 20,000 paths per start, $\widehat\Sigma^{\rm av}=0.6599$. Simultaneous 95\% DKW bound 0.0126.}
\label{tab:asgdm}
\begin{center}
\begin{tabular}{lrrrrr}
\toprule
$\theta_0$ & $n=1{,}000$ & $n=3{,}000$ & $n=10{,}000$ & $n=30{,}000$ & $n=100{,}000$ \\
\midrule
$-3$ & 0.284 & 0.094 & 0.023 & 0.010 & 0.006 \\
$1$ & 0.047 & 0.031 & 0.016 & 0.009 & 0.009 \\
$3$ & 0.281 & 0.099 & 0.030 & 0.014 & 0.005 \\
\bottomrule
\end{tabular}
\end{center}
\end{table}

\subsection[Quadratic transient comparisons]{Quadratic transient comparisons}
\label{app:numerical-quadratic}
\label{app:gaussian-illustration}

Figure~\ref{fig:numerical-appendix}(b)--(c) supplements the slow-coordinate comparison in Figure~\ref{fig:numerical-report}(c).
We use $g(\theta,x_k)=A\theta+\varepsilon_k$, with $A=\diag(\mu,L)$, iid $\varepsilon_k\sim N(0,I_2)$, and $y_0=0$. Each method uses its own normalized learning rate $\eta$ and reference covariance $\eta\Gamma$, where $\Gamma=A^{-1}/2$. Effective SGD means SGD at step $\eta=\gamma/(1-\beta)$ of the SGDM method being compared.

For Figure~\ref{fig:numerical-report}(c), $(\mu,L)=(1,100)$ and $\theta_0=(1,1)^\top/\sqrt2$. We compare optimal SGDM at $(\gamma,\beta)=(4/121,81/121)$ with optimal SGD at step $2/101$; optimal refers to the quadratic contraction factor. The normalized learning rate is $\gamma/(1-\beta)$ for SGDM and the step itself for SGD. With $Z_{n,1}=\theta_{n,1}/\sqrt{\eta\Gamma_{11}}$, define
\[
 D_{\rm K}(n)=\sup_z\left|\mathbb P(Z_{n,1}\le z)-\Phi(z)\right|,
\]
where $\Phi$ is the standard normal CDF. The curves are computed by propagating Gaussian means and covariances.

Panel (b) sets $(\mu,L)=(1,80)$, $\beta=0.5$, $\gamma=0.01125$, and $\theta_0=(1,0)^\top$. SGDM and effective SGD both have $\eta=0.0225$ and are stable by~\eqref{eq:exact-quadratic-stability}. Their slow-coordinate distances $D_{\rm K}(n)$ are close over $0\le n\le1200$.

Panel (c) uses the optimal tunings and initialization of Figure~\ref{fig:numerical-report}(c), and adds effective SGD at step $0.1$. To compare the joint laws, define
{
\[
 D_{\rm H}(n)=\sup_{|u|=1,\,t\in\mathbb R}\left|\mathbb P\left(u^\top(\eta\Gamma)^{-1/2}\theta_n\le t\right)-\Phi(t)\right|.
\]
}
We propagate Gaussian means and covariances and approximate the directional supremum using 721 equally spaced angles in $[0,\pi]$. These calculations use no Monte Carlo sampling. At $n=100$, the distances are $0.3995$ for optimal SGDM and $0.7644$ for optimal SGD. Effective SGD remains near one because its $L$-coordinate multiplier is $1-0.1L=-9$. The SGDM plateau reflects its fixed-step stationary covariance: after scaling by $(\eta\Gamma)^{-1/2}$, this covariance is $\diag(1.01,101)$. Increasing $n$ at these steps does not remove the covariance difference. The small-step limit instead sends $\gamma$ to zero after burn-in.

\subsection[Small-step last-iterate limit]{Small-step last-iterate limit}
\label{app:numerical-last-iterate}
\label{app:non-gaussian-illustration}

Figure~\ref{fig:numerical-appendix}(d) and Table~\ref{tab:student_t} give the linear Student-$t_3$ illustration of the small-step last-iterate limit.
Take $g(\theta,\varepsilon)=\theta+\varepsilon$, with iid $\varepsilon=T_3/\sqrt3$, $\theta_0=1$, $y_0=0$, and $\beta=0.5$. The noise has mean zero and variance one. Here $\theta^\star=0$, $\mu=L_2=S=1$, and $\Gamma=1/2$, so Theorem~\ref{thm:last-iterate-clt} gives a standard normal limit for $Z=\theta_n/\sqrt\gamma$. We use 50,000 independent paths at each of $\gamma=0.1,0.05,0.025,0.0125$, with
\[
 n=\left\lceil\frac{4\log(1/\gamma)}{\gamma}\right\rceil.
\]
This choice satisfies the burn-in condition~\eqref{eq:last-iterate-burnin} as $\gamma\downarrow0$.

Panel (d) plots $\widehat F(z)-\Phi(z)$, where $\widehat F$ is the empirical CDF of $Z$. The four step sizes use independent runs. The plotted curves use 601 points in $[-3,3]$, while the table distances use both sides of the empirical CDF at every sorted sample value. In Table~\ref{tab:student_t}, $\widehat D_0$, $\widehat D_F$, and $\widehat D_{\rm mom}$ are the Kolmogorov distances to $N(0,1)$, $N(0,F)$, and $N(m,v)$, respectively. Here
\[
 F=\frac{3}{3-\gamma},\qquad m=\mathbb EZ,\qquad v=\operatorname{Var}(Z),
\]
with $F$ the exact scaled stationary variance and $(m,v)$ computed from the linear recursion. At these settings, $|m|<1.1\times10^{-12}$ and $v$ agrees with $F$ to fourteen decimal places. The distances to $N(0,F)$ and $N(m,v)$ agree at the displayed precision. They decrease as $\gamma$ decreases; the smallest remains $0.0184$, above the simultaneous 95\% empirical-CDF error bound $0.0071$ for the four settings.
\begin{table}[!htbp]
\caption{Student-$t_3$ last iterate $\theta_n/\sqrt\gamma$, $g=\theta+\varepsilon$, $\varepsilon=T_3/\sqrt3$, $\beta=0.5$, 50,000 paths. Var is the exact finite-$n$ variance. Simultaneous 95\% DKW bound 0.0071.}
\label{tab:student_t}
\begin{center}
\begin{tabular}{rrrrrr}
\toprule
$\gamma$ & $n$ & Var & $\widehat D_0$ & $\widehat D_F$ & $\widehat D_{\rm mom}$ \\
\midrule
0.1 & 93 & 1.0345 & 0.0335 & 0.0370 & 0.0370 \\
0.05 & 240 & 1.0169 & 0.0269 & 0.0289 & 0.0289 \\
0.025 & 591 & 1.0084 & 0.0224 & 0.0234 & 0.0234 \\
0.0125 & 1403 & 1.0042 & 0.0179 & 0.0184 & 0.0184 \\
\bottomrule
\end{tabular}
\end{center}
\end{table}

\subsection[Stochastic-curvature calculations]{Stochastic-curvature calculations}
\label{app:numerical-curvature}

Table~\ref{tab:curvature} compares the exact stationary and averaged limiting variances with Monte Carlo calculations for the model in Figure~\ref{fig:numerical-report}(d).
The model and parameters are those of Section~\ref{sec:numerical-contraction}.
We use $H_{i,k}=\lambda_i(1+\sqrt\nu\,\xi_{i,k})$, with independent Rademacher $\xi_{i,k}$ and additive noise $N(0,I_2)$. Here $(\mu,L)=(0.01,1)$ and $\nu=0.005$, so $\kappa=100$.
\begin{table}[!htbp]
\caption{Exact stationary and averaged limiting variances in the $L$ direction against Monte Carlo (400,000 chains, $n=20000$; standard error in parentheses).}
\label{tab:curvature}
\begin{center}
\begin{tabular}{lrrrrr}
\toprule
$(\gamma,\beta)$ & $1-q$ & $V_L$ & $\widehat V_L$ & $\Sigma^{\rm av}_L$ & $\widehat\Sigma^{\rm av}_L$ \\
\midrule
$(1,0.00)$ & 0.0100 & 1.0050 & 1.0050 & 1.0050 & 1.0077 (0.0023) \\
$(1,0.50)$ & 0.0204 & 1.5113 & 1.5113 & 1.0076 & 1.0085 (0.0023) \\
$(1,0.81)$ & 0.0952 & 3.7033 & 3.7034 & 1.0185 & 1.0169 (0.0023) \\
\bottomrule
\end{tabular}
\end{center}
\end{table}

The grid has 200 logarithmically spaced $\gamma$ values in $[10^{-3},4]$ and 200 linearly spaced $\beta$ values in $[0,0.99]$, including the endpoints. Adding $\gamma=1$ and $\beta=0.81$ gives a $201\times201$ grid. The grid contains 38,618 stable tunings, with forgetting factors computed from~\eqref{eq:curvature-moment-operator} and variances from~\eqref{eq:curvature-variances}. Stability means $q<1$, with no numerical tolerance. Among stable grid points, the numerical lower envelope retains a point only when its variance ratio is strictly smaller than that of every faster point. It contains 4,643 points; no optimization outside the grid is used. The lower boundary in Figure~\ref{fig:numerical-report}(d) interpolates the numerical lower envelope and reference point $(0,1)$. The filled area does not certify that every interior point is attained by a stable tuning. The star marks the explicit tuning~\eqref{eq:curvature-retuning}. At $(\gamma,\beta)=(1,0.81)$, it gives $1-q\simeq0.0952$ and a variance ratio of about $1.0185$.

We use 400,000 chains, a burn-in of 1000 steps, and averaging length $n=20000$ at $(\gamma,\beta)=(1,0),\allowbreak(1,0.5),\allowbreak(1,0.81)$. Each tuning uses an independent run of the $L$ coordinate, starting at $e_0=e_{-1}=0$. With $B=1000$, $M=400000$, and $S_{n,j}=\sum_{k=B+1}^{B+n}e_k^{(j)}$, the estimates are
\[
 \widehat V_L=\frac1{nM}\sum_{k=B+1}^{B+n}\sum_{j=1}^{M}\left(e_k^{(j)}\right)^2,\qquad
 \widehat\Sigma_L^{\rm av}=\frac1M\sum_{j=1}^{M}\frac{S_{n,j}^2}{n}.
\]
Both use the known stationary mean zero without subtracting a sample mean. The Monte Carlo columns estimate $V_L$ and $\Sigma_L^{\rm av}$, with standard errors for the latter in parentheses. These standard errors are the sample standard deviations of $S_{n,j}^2/n$, using divisor $M-1$, divided by $\sqrt M$. The estimates of the averaged limiting variance agree with the formulas at this error scale. The table checks the variance formulas at three tunings; it does not establish optimality of the numerical lower envelope or the finite-averaging error rate.

\end{document}